\documentclass[a4paper,11pt]{article}
\usepackage[
margin=3cm,
includefoot,
footskip=12pt,
]{geometry}
\usepackage{xspace}
\usepackage{tikz}
\usepackage{color}
\definecolor{antiquebrass}{rgb}{0.8, 0.58, 0.46}
\usepackage{algpseudocode,algorithm}

\algrenewcommand\algorithmicrequire{\textbf{Input:}}
\algrenewcommand\algorithmicensure{\textbf{Output:}}
\algnewcommand{\Initialize}[1]{%
  \State \textbf{Initialize:}
  \Statex \hspace*{\algorithmicindent}\parbox[t]{.8\linewidth}{\raggedright #1}
}
\algdef{SE}[FUNCTION]{Function}{EndFunction}%
   [2]{\algorithmicfunction\ \textproc{#1}\ifthenelse{\equal{#2}{}}{}{(#2)}}%
   {\algorithmicend\ \algorithmicfunction}%

\renewcommand{\Return}{\State \textbf{return} }
\usepackage{epsfig}
\usepackage{subfig}
\usepackage{mathrsfs,amsmath,amsthm,amssymb,mdwlist,bbm}

\usepackage{makecell}
\usepackage{multirow}

\allowdisplaybreaks

\renewcommand{\ge}{\geqslant}
\renewcommand{\le}{\leqslant}

\usepackage{enumerate}

\usepackage[shortlabels]{enumitem}

\newtheorem{lemma}{Lemma}
\newtheorem{definition}{Definition}

 \setlist{nolistsep,leftmargin=*}

\usepackage{nicefrac}

\newcommand{\eps}{\ensuremath{\varepsilon}\xspace}
\renewcommand{\epsilon}{\eps}
\let\mydelta\delta
\renewcommand{\delta}{\ensuremath{\mydelta}\xspace}
\let\myalpha\alpha
\renewcommand{\alpha}{\ensuremath{\myalpha}\xspace}
\let\mystar\star
\renewcommand{\star}{\ensuremath{\mystar}}

\newcommand{\NP}{\ensuremath{\mathsf{NP}}\xspace}

\newcommand{\Pb}{\ensuremath{\mathsf{P}}\xspace}

\newcommand{\NPH}{\ensuremath{\mathsf{NP}}-hard\xspace}

\renewcommand{\AA}{\ensuremath{\mathcal A}\xspace}
\newcommand{\BB}{\ensuremath{\mathcal B}\xspace}
\newcommand{\CC}{\ensuremath{\mathcal C}\xspace}
\newcommand{\DD}{\ensuremath{\mathcal D}\xspace}
\newcommand{\EE}{\ensuremath{\mathcal E}\xspace}
\newcommand{\FF}{\ensuremath{\mathcal F}\xspace}
\newcommand{\GG}{\ensuremath{\mathcal G}\xspace}
\newcommand{\HH}{\ensuremath{\mathcal H}\xspace}
\newcommand{\II}{\ensuremath{\mathcal I}\xspace}
\newcommand{\JJ}{\ensuremath{\mathcal J}\xspace}
\newcommand{\KK}{\ensuremath{\mathcal K}\xspace}
\newcommand{\LL}{\ensuremath{\mathcal L}\xspace}
\newcommand{\MM}{\ensuremath{\mathcal M}\xspace}

\newcommand{\OO}{\ensuremath{\mathcal O}\xspace}

\newcommand{\RR}{\ensuremath{\mathcal R}\xspace}
\renewcommand{\SS}{\ensuremath{\mathcal S}\xspace}

\newcommand{\UU}{\ensuremath{\mathcal U}\xspace}
\newcommand{\VV}{\ensuremath{\mathcal V}\xspace}
\newcommand{\WW}{\ensuremath{\mathcal W}\xspace}
\newcommand{\XX}{\ensuremath{\mathcal X}\xspace}
\newcommand{\YY}{\ensuremath{\mathcal Y}\xspace}
\newcommand{\ZZ}{\ensuremath{\mathcal Z}\xspace}

\newtheorem{theorem}{\bf Theorem}

\newtheorem{corollary}{\bf Corollary}

\usepackage{cleveref}

\crefname{theorem}{Theorem}{Theorems}
\crefname{observation}{Observation}{Observations}
\crefname{lemma}{Lemma}{Lemmas}
\crefname{corollary}{Corollary}{Corollaries}
\crefname{proposition}{Proposition}{Propositions}
\crefname{example}{Example}{Examples}
\crefname{claim}{Claim}{Claims}
\crefname{table}{Table}{Tables}
\crefname{equation}{Inequality}{Inequalities}
\crefname{reductionrule}{Reduction rule}{Reduction rules}
\crefname{section}{Section}{Sections}

\newcommand{\edWD}{\ensuremath{(\eps,\delta)-}{\sc Winner Prediction}\xspace}
\newcommand{\dWD}{\ensuremath{\delta-}{\sc Winner Prediction}\xspace}

\newcount\Comments  % 0 suppresses notes to selves in text
\newcommand{\kibitz}[2]{\ifnum\Comments=1{\color{#1}{#2}}\fi}

\title{Sampling-Based Winner Prediction in District-Based Elections}
\author{Palash Dey \and Debajyoti Kar \and Swagato Sanyal\\\texttt{palash.dey@cse.iitkgp.ac.in,debajyoti.kar@iitkgp.ac.in,}\\\texttt{swagato@cse.iitkgp.ac.in}\\IIT Kharagpur, India}
\date{}

\begin{document}

\maketitle
\begin{abstract}
    In a district-based election, we apply a voting rule $r$ to decide the winners in each district, and a candidate who wins in a maximum number of districts is the winner of the election. We present efficient sampling-based algorithms to predict the winner of such district-based election systems in this paper. When $r$ is plurality and the margin of victory is known to be at least \eps fraction of the total population, we present an algorithm to predict the winner. The sample complexity of our algorithm is $\OO\left(\frac{1}{\eps^4}\log \frac{1}{\eps}\log\frac{1}{\delta}\right)$. We complement this result by proving that any algorithm, from a natural class of algorithms, for predicting the winner in a district-based election when $r$ is plurality, must sample at least $\Omega\left(\frac{1}{\eps^4}\log\frac{1}{\delta}\right)$ votes. We then extend this result to any voting rule $r$. Loosely speaking, we show that we can predict the winner of a district-based election with an extra overhead of $\OO\left(\frac{1}{\eps^2}\log\frac{1}{\delta}\right)$ over the sample complexity of predicting the single-district winner under $r$. We further extend our algorithm for the case when the margin of victory is unknown, but we have only two candidates. We then consider the median voting rule when the set of preferences in each district is single-peaked. We show that the winner of a district-based election can be predicted with $\OO\left(\frac{1}{\eps^4}\log\frac{1}{\eps}\log\frac{1}{\delta}\right)$ samples even when the harmonious order in different districts can be different and even unknown. Finally, we also show some results for estimating the margin of victory of a district-based election within both additive and multiplicative error bounds.

% Predicting the winner of an election and estimating the margin of victory of that election are two very important problems both for media agents and computational social choice theorists. Since it is often infeasible to elicit the preferences of all the voters, a common strategy is to sample a small subset of the votes and predict the winner based on this sampled set of votes with high confidence. Previous works \cite{dey2015sample} \cite{dey2015estimating} \cite{bhattacharyya2021predicting} have studied the above problem and attempted to provide upper and lower bounds on the sample complexity for various homogeneous voting rules. We study two of the popularly used voting rules - the plurality rule and the median rule. We significantly generalize the plurality rule to the case where the population has been grouped into multiple districts. We also present efficient algorithms for estimating the Margin of Victory of elections.
\end{abstract}

\section{Introduction}

Voting and election serve as one of the most popular methodologies to aggregate different preferences, eventually choosing one of many candidate options. In political elections, one of the hottest questions for NEWS media and many other people is who will win in the upcoming election~\cite{perse2016media}? To predict the winner of an upcoming election, a pollster typically samples some votes with the hope that the sampled votes will help him/her correctly predict the winner. However, sampling votes, depending on the sampling requirement and procedure, typically involves substantial cost. Hence, a natural goal of the pollster is to minimize the cost, which often translates to minimizing the number of samples, without compromising the quality (or success rate) of prediction. Intuitively speaking, this is the winner prediction problem, which is the main focus of our paper.

The same winner prediction problem becomes relevant not only for predicting the winner of an upcoming election, but also in many other applications like social surveys, post election audit, etc. Organizations and companies, for example, often conduct various surveys to predict the success of their products which they are planning to manufacture. We carry out post election audits on paper ballots to check if there are any human or machine-related errors in the election process~\cite{norden2007post,stark2008conservative,hall2009implementing,rivest2012bayesian,wolchok2010security}.

Bhattacharyya and Dey resolved the sample-complexity of the winner prediction problem for many popular voting rules, for example, $k$-approval, Borda, approval, maximin, simplified Bucklin, and plurality with run off~\cite{DBLP:journals/ai/BhattacharyyaD21}. A voting rule is a function which selects one winner from a set of votes. We refer to the chapter by Zwicker for an introduction to voting and some common voting rules ~\cite{DBLP:reference/choice/Zwicker16}. However, Bhattacharyya and Dey only considered single district elections whereas many real-world election systems, especially political elections in many countries, for example, US Presidential election, Indian general elections, etc. are district based. In a district-based election system, the voters are partitioned into districts. We use some voting rule $r$, the plurality voting rule for US Presidential election and Indian general elections, to select a winner in each district. The candidate (for US Presidential election) or the political party (Indian general elections) who wins in a maximum number of districts is declared as the winner of the election. In the plurality voting system, each voter votes for one of the candidates and the candidate who receives the maximum number of votes is declared as the winner. We study the winner prediction problem for district-based elections in this paper.

\subsection{Our Contribution}

The primary focus of our paper is the \edWD problem, which is defined as follows. 

\begin{definition}[\edWD]
Given an election $E$ with $N$ voters partitioned into $k$ districts where a voting rule $r$ is used to determine the winner of each district, and whose margin of victory is at least $\epsilon N$, compute the winner of the election with probability at least $1-\delta$.
\end{definition}

Our specific contributions are the following. If not mentioned otherwise, we use the plurality voting rule to select the winners in each district.

\begin{enumerate}
    \item We design an algorithm for \edWD with sample complexity $\OO\left(\frac{1}{\epsilon^4}\log \frac{1}{\epsilon}\log \frac{1}{\delta}\right)$ [\Cref{alg:alg1}]. We partially complement this result by showing that any algorithm for \edWD that works by first sampling $l_1$ districts uniformly at random with replacement and then sampling $l_2$ votes uniformly at random with replacement from each of the sampled districts, must satisfy $l_1 = \Omega \left(\frac{1}{\epsilon^2}\log \frac{1}{\delta}\right )$ and $l_2 = \Omega\left (\frac{1}{\epsilon^2}\right)$ even when there are only $2$ candidates and all the districts have equal population [\Cref{thm:epsfourlowerbound}].
    
    \item  We then generalize our above result to any arbitrary voting rule $r$ in each district. Let $\chi_r (m,\epsilon,\delta)$ be the number of samples required so that the predicted winner of a single-district election using using rule $r$ with $n$ voters and $m$ candidates, can be made winner by changing at most $\epsilon n$ votes. Then, using the prediction algorithm for $r$, we design an algorithm for \edWD for $r$ with sample complexity $\OO \left(\frac{1}{\epsilon^2}\log \frac{1}{\delta}\cdot \chi_r (m,\epsilon,\epsilon)\right )$ [\Cref{thm:generalizationofplurality}].
\end{enumerate}

In \edWD, we assume that we know some lower bound on the margin of victory of the election. Obviously this information may not always be available. To cater those situations, we define and study the \dWD problem.

\begin{definition}[\dWD]
Given an election $E$ with $N$ voters partitioned into $k$ districts where a voting rule $r$ is used to determine the winner of each district, compute the winner of the election with probability at least $1-\delta$.
\end{definition}

We note that we have no information on the margin of victory of the election in \dWD.

\begin{enumerate}
    \item We design an algorithm for \dWD with sample complexity $\OO \left(\frac{1}{\epsilon^4}\log^2 \frac{1}{\epsilon \delta}\right)$ when we have only $2$ candidates and the number of voters in each district is at most a constant times the average population of a district [\Cref{alg:alg3}].
    
    \item For arbitrary number of voters in each district, we design an algorithm for \dWD with sample complexity $\OO\left(\frac{1}{\epsilon^6}\log^2 \frac{1}{\epsilon \delta}\right)$ when we have only $2$ candidates [\Cref{alg:alg4}].
\end{enumerate}

We next study the case when median rule is used to decide the winner in each district. The harmonious order with respect to which median rule is used can be different in different districts and may or may not be known. If the harmonious order is unknown in a district, we make the assumption that the preference profile of each voter in that district is single-peaked. We design an algorithm for \edWD for this case with sample complexity $\OO \left(\frac{1}{\epsilon^4}\log \frac{1}{\epsilon}\log \frac{1}{\delta}\right)$ [\Cref{thm:winnerdetmedianrule,thm:winnerdetmedianrule1,cor:mediancorollary}]. 

In all of the above algorithms, we assumed that we were allowed to get uniform random samples from the population. Obviously this might not be the case. We therefore define and study the $(\epsilon,\delta,\gamma)-${\sc{Winner-Determination}} problem and its related multiple-district variant, the $(\epsilon,\delta,\gamma)-${\sc{Winner-Prediction}} problem. 

\begin{definition}[$(\epsilon,\delta,\gamma)-${\sc{Winner-Determination}}]
Given an election $E$ whose margin of victory is at least $\epsilon N$ and an unknown distribution $U$ over the voters such that $d_{\mathrm{TV}}(U,V)\le \gamma$ where $\gamma = o(\eps)$ (here $V$ denotes the uniform distribution over the voters), determine the winner of the election with probability at least $1-\delta$.
\end{definition}

Here $d_{\mathrm{TV}}(U,V)$ is the total variational distance between the distributions $U$ and $V$. 

\begin{definition}[$(\epsilon,\delta,\gamma)-${\sc{Winner-Prediction}}]
Given an election $E$ with $N$ voters partitioned into $k$ districts where a voting rule $r$ is used to determine the winner of each district, and unknown distributions $U_j$, $j\in [k]$ over the voters in each district and $U$ over the districts such that $d_{\mathrm{TV}}(U_j,V_j), d_{\mathrm{TV}}(U,V)\le \gamma$, where $\gamma = o(\eps)$ (here $(V_j)_{j\in [k]}, V$ denote uniform distributions over the voters in each district and over the districts respectively). Also given that $\text{MOV}(E) \ge \eps N$, determine the winner of the election with probability at least $1-\delta$.
\end{definition}

We restrict our attention to the plurality rule and present algorithms with sample complexities $\OO\left(\frac{1}{(\eps - \gamma})^2)\log \frac{1}{\delta}\right)$ [\Cref{lem:pluralitywithnoise}] and $\OO\left(\frac{1}{(\eps-\gamma)^4}\log \frac{1}{\eps}\log \frac{1}{\delta}\right)$ [\Cref{lem:pluralitywithnoisemultiple}] for $(\epsilon,\delta,\gamma)-${\sc{Winner-Determination}} and $(\epsilon,\delta,\gamma)-${\sc{Winner-Prediction}} respectively. 

Last but not the least, we study the problem of estimating the margin of victory of a district-based election within additive and multiplicative error bounds. We define the following two problems.

\begin{definition}[$(\eps,\delta)-${\sc{MOV-Additive}}]
Given an election $E$ with $N$ voters partitioned into $k$ districts where a voting rule $r$ is used to determine the winner of each district, estimate the margin of victory of $E$ within an additive $\eps N$ error with probability at least $1-\delta$.
\end{definition}

\begin{definition}[$(\eps,\delta)-${\sc{MOV-Multiplicative}}]
Given an election $E$ with $N$ voters partitioned into $k$ districts where a voting rule $r$ is used to determine the winner of each district, estimate the margin of victory of $E$ within a multiplicative error of $1 \pm \eps$ with probability at least $1-\delta$.
\end{definition}

\begin{enumerate}
    \item We design an algorithm for $(\eps,\delta)-${\sc{MOV-Additive}} with sample complexity $\OO\left(\frac{1}{\eps^6}\log \frac{1}{\eps \delta}\log \frac{1}{\delta}\right)$ when we have only 2 candidates and the number of voters in each district is at most a constant times the average population of a district [\Cref{thm:movestimatefordistrictplurality}].
    \item For $(\eps,\delta)-${\sc{MOV-Multiplicative}}, we present an algorithm with expected sample complexity $\OO\left(\frac{1}{\epsilon^7}\frac{1}{\gamma^6}\left(\frac{1}{\epsilon}\log \frac{1}{\epsilon \gamma}+ \log \frac{1}{\delta}\right)^2\right)$ when there are 2 candidates and the population of each district is bounded by a constant times the average population of a district, where $\gamma N$ is the true (unknown) margin of victory of the election [\Cref{thm:lastthmijcai}].
\end{enumerate}

\iffalse
We design sampling based algorithms for estimating the margin of victory within (i) an additive error of \eps times the number of voters with sample complexity $\OO \left(\frac{1}{\epsilon^6}\log \frac{1}{\epsilon \delta}\log \frac{1}{\delta}\right)$ [\Cref{thm:movestimatefordistrictplurality}], and (ii) a multiplicative error of \eps with sample complexity $\OO \left(\frac{1}{\epsilon^7}\frac{1}{\gamma^6}\left(\frac{1}{\epsilon}\log \frac{1}{\epsilon \gamma}+ \log \frac{1}{\delta}\right)^2\right)$ [\Cref{thm:lastthmijcai}] (here $\gamma N$ is the true margin of victory of the election).
\fi

In summary, our main contribution is to initiate the study of sample complexity for predicting winner in district-based elections. We believe that these problems and our preliminary results are practically important as well as theoretically interesting.

\subsection{Related Work}

The most immediate predecessor of our \edWD problem is the work of Bhattacharyya and Dey who worked on the same problem but focused only on single district elections~\cite{DBLP:journals/ai/BhattacharyyaD21}. Another classical problem which is related to our problem, is the winner determination problem in computational social choice. Here, we are given a set of votes, and we need to compute the winner of these votes under some voting rule. Bartholdi et al. were the first to observe that there are popular voting rules, namely the Kemeny voting rule, for which, determining a winner is \NPH\cite{bartholdi1989voting}. Hemaspaandra et al. later settled the complexity of the winner determination problem for the Kemeny voting rule by showing that the problem is complete for the complexity class $\Pb^\NP_{||}$~\cite{hemaspaandra2005complexity}. Similar results hold for the Dodgson and Young voting rules also~\cite{hemaspaandra1997exact,rothe2003exact,brandt2015bypassing,hemaspaandra2009hybrid}. The main difference between our work and the above papers on the winner determination problem is that we focus on sample complexity, whereas they focus on time complexity.

Our problem is also closely related to the general question: do we need to see all the votes to determine the winner? Conitzer and Sandholm developed preference elicitation policies as a sequence of questions posed to the voters~\cite{conitzer2002vote}. They showed that finding an effective elicitation policy is \NPH even for some common voting rules. On the positive side, many effective elicitation policies have been subsequently developed for many important restricted domain and settings~\cite{conitzer2009eliciting,DBLP:conf/ijcai/DeyM16,DBLP:conf/ijcai/DeyM16a,ding2012voting,lu2011robust,lu2011vote,oren2013efficient}.

\section{Preliminaries}
\label{sec:preliminaries}
We now define an election $E$. Let $\VV$ be a set of $N$ voters and $C$ be a set of $m$ candidates. The vote of each voter $v\in \VV$ is a complete order over the set of candidates. Let $\LL(C)$ denote the set of all complete orders over $C$. Thus $\LL(C)^N$ denotes the set of all preference profiles of the $N$ voters. A map $r \colon \LL(C)^N \rightarrow C$ is called a \textit{voting rule}. Throughout we assume that there is an arbitrary but fixed rule for resolving ties. For any $a\in \LL(C)$, let $s(a)$ denote the most preferred candidate in $a$. A voting rule $r$ is said to be \textit{top-ranked} if $r(a_1,\ldots,a_N)=r(b_1,\ldots,b_N)$ whenever $s(a_i)=s(b_i)$, $\forall i \in [N]$, i.e the winner of the election depends only the most preferred candidates of the $N$ voters. For a top-ranked voting rule, we say that a candidate $x\in C$ \textit{receives} vote $a\in \LL(C)$ if $x=s(a)$. We study two top-ranked voting rules - the plurality rule and the median rule.

Given an election $E$, for any two candidates $x,y \in C$, let $\pi_E(x,y)$ denote the number of voters who prefer $x$ to $y$. Define $\rho_E(x,y)=\pi_E(x,y)-\pi_E(y,x)$. Then a candidate $x$ is called the \textit{Condorcet winner} of the election if $\rho_E(x,y)>0$, $\forall y \in C\setminus \{x\}$. The Condorcet winner, if exists, is unique.

The \textit{Margin Of Victory} (MOV) of an election $E$, denoted by $\text{MOV}(E)$, is defined as the minimum number of votes to be altered so as to change the winner of the election.

Bhattacharyya and Dey introduced the $(\epsilon,\delta)-${\sc{Winner-Determination}} problem in \cite{dey2015sample} as follows:

\begin{definition}[$(\epsilon,\delta)-${\sc{Winner-Determination}}]
Given an election $E$ whose margin of victory is at least $\epsilon N$, determine the winner of the election with probability at least $1-\delta$.
\end{definition}

They established upper and lower bounds for various homogeneous voting rules. Another related work is by Dey and Narahari \cite{dey2015estimating} where they study the $(c,\epsilon,\delta)-${\sc{Margin of Victory}} problem.

\begin{definition}[$(c,\epsilon,\delta)-${\sc{Margin of Victory}}]
\label{defn:defitiibb}
Given an election $E$, determine $\text{MOV}(E)$ with an additive error of at most $c\text{MOV}(E) + \epsilon N$ with probability at least $1-\delta$.
\end{definition}

We repeatedly use the following concentration bounds.

\begin{theorem}[\cite{bhattacharyya2021predicting}] [Chernoff Bound]
	\label{thm:chernoff}
	Let $X_1,\ldots X_l$ be a sequence of $l$ independent 0-1 random variables (not necessarily identical). Let $X=\sum\limits_{i=1}^{l} X_i$ and $\mu = \mathbb{E}[X]$. Then for any $\theta \ge 0$, 
	\begin{enumerate}
		\item Additive form: $\text{Pr}(|X-\mu| \ge \theta l) \le 2e^{-2\theta^2 l}$.
		\item Multiplicative form: $\text{Pr}(|X-\mu| \ge \theta \mu) \le 2e^{-\frac{\theta^2}{3}\mu}$.
	\end{enumerate}
\end{theorem}

\begin{theorem}[\cite{hoeffding1994probability}]  [Hoeffding's Inequality]
	\label{thm:hoeffding}
	Let $X_1,\ldots,X_l$ be a sequence of independent and identically distributed random variables such that $X_i \in [a,b]$, $\forall i \in [l]$, for some real numbers $a<b$. Let $\overline{X}=\frac{\sum\limits_{i=1}^{l} X_i}{l}$. Then for any $\theta \ge 0$, $\text{Pr}(|\overline{X}-\mathbb{E}[\overline{X}]|\ge \theta)\le 2e^{-\frac{2\theta^2}{(b-a)^2}l}$.
\end{theorem}

\section{Winner Prediction for Plurality}
\label{sec:plurality}
For each candidate $x\in C$, let $g(x)$ denote the number of votes where $x$ is most preferred. Then the single-district plurality rule declares a candidate $x$ with the highest value of $g(x)$, as the winner. Since plurality is a top-ranked voting rule, each vote can also be viewed as a single candidate.

We now introduce some notations. Given any list $L = (x_1,\ldots,x_t)$ of candidates, let $\text{MAJ}(L)$ (resp. $\text{SEC-MAJ}(L)$) denote the candidate with the largest (resp. second largest) frequency in $L$ (tie-breaking rule is arbitrary but fixed with respect to some arbitrary but fixed rule). We state some of the known results on upper and lower bounds on sample complexity. Let $E$ be an election where the single-district plurality rule is used to decide the winner. The result below is a slight modification of Theorem 7 in \cite{dey2015sample}. %We include a proof in \Cref{sec:omittedproofoftheorem} for completeness.
\begin{theorem}[\cite{dey2015sample}]
\label{thm:pluralitythm} 
If $\frac{3}{\vartheta^2}\log \frac{2}{\delta}$ votes are sampled uniformly at random with replacement, then with probability at least $1-\delta$, for every candidate, the fraction of sampled votes received differs from the true fraction of votes received by less than $\vartheta$, for any $\vartheta > 0$. 
\end{theorem}
\begin{proof}
We will need the following lemma.

\begin{lemma} [\cite{dey2015sample}]
\label{lem:ineqaulity}
 Let $f \colon \mathbb{R}^+ \rightarrow \mathbb{R}^+$ be defined by $f(x) = e^{-\frac{\lambda}{x}}$. Then 
\[ f(x) + f(y) \le f(x-h) + f(y+h) \]
whenever $x,y,h > 0$, $\frac{\lambda}{x+y}>2$ and $h\le x<y$.
\end{lemma}

For each candidate $x$, let $X^x_i$ be the random variable indicating whether $x$ receives the $i^{\text{th}}$ sampled vote. Then $X^x=\sum\limits_{i=1}^{l} X^x_i$ denotes the number of sampled votes received by $x$. Let $\hat{g}(x) = \frac{N}{l}\cdot X^x$ denote the predicted number of votes of candidate $x$. Thus $\text{Pr}(\left|\hat{g}(x) - g(x)\right|\ge \vartheta N) = \text{Pr}(\left|\frac{l}{N}\cdot \hat{g}(x) - \frac{l}{N}\cdot g(x)\right|\ge \frac{\vartheta N}{g(x)}\cdot \frac{lg(x)}{N}) \le 2e^{-\frac{\vartheta^2lN}{3g(x)}}$. The final inequality follows by applying the multiplicative form of Chernoff bound (\Cref{thm:chernoff}) with $\theta = \frac{\vartheta N}{g(x)}$.

By union bound, $\text{Pr}(\exists x\in C, \left|\hat{g}(x)-g(x)\right| \ge \vartheta N)\le \sum\limits_{x\in C} 2e^{-\frac{\vartheta^2lN}{3g(x)}} \le 2e^{-\frac{\vartheta^2l}{3}} = \delta$. The second inequality follows from Lemma \ref{lem:ineqaulity}: since $g(x)\in [0,N]$ $\forall x\in C$, and $\sum\limits_{x\in C} g(x) = N$, $\sum\limits_{x\in C} 2e^{-\frac{\vartheta^2lN}{3g(x)}}$ is maximised when $g(x)=N$ for some candidate $x$ and $g(y)=0$, $\forall y \in C\setminus \{x\}$. Thus with probability at least $1-\delta$, for each candidate, the predicted number of votes differs from the true number of votes received by less than $\vartheta N$.
\end{proof}

\begin{corollary}
\label{cor:pluralitycor}
If $\text{MOV}(E)\ge \epsilon N$, then $\frac{3}{\epsilon^2}\log \frac{2}{\delta}$ samples are enough to predict the winner correctly with probability at least $1-\delta$.
\end{corollary}

\iffalse
\begin{theorem}[\cite{dey2015estimating}]
\label{thm:estimatemov}
There exists an algorithm that uses $l=\frac{12}{\epsilon^2} \log \frac{2}{\delta}$ samples and estimates $\text{MOV}(E)$ within an additive error of $\epsilon N$ with probability at least $1-\delta$.
\end{theorem}
\fi

\begin{theorem}[\cite{dey2015sample,canetti1995lower,bar2001sampling}]
\label{thm:lowerboundonplurality} For $\epsilon \le \frac{1}{8}$ and $\delta \le \frac{1}{6}$, every $(\epsilon,\delta)-${\sc{Winner-Determination}} algorithm needs at least $\frac{1}{4\epsilon^2}\log \frac{1}{8e\sqrt{\pi}\delta}$ samples for any voting rule that reduces to the single-district plurality rule for 2 candidates.
\end{theorem}

\iffalse
\begin{theorem}[\cite{dey2015sample,canetti1995lower,bar2001sampling}]
\label{thm:lowerboundmov}
For any $c\in [0,1)$, every $(c,\epsilon,\delta)-${\sc{margin of victory}} algorithm needs at least $\frac{(1-c)^2}{36\epsilon^2}\log \frac{1}{8e\sqrt{\pi}\delta}$ samples for all voting rules which reduce to the plurality rule for 2 candidates.
\end{theorem}
\fi

%\subsection{District level plurality election}
\label{sec:pluralitysubsec}
We now generalize the above setting to the case where there are multiple districts. Let $D=\{d_1,\ldots,d_k\}$ be a set of $k$ districts where district $d_j$ has population $n_j$ and $N=\sum_{j=1}^{k} n_j$ is the total population. The winner of each district is decided using the single-district plurality rule. A candidate winning in maximum number of districts is declared as the overall winner of the election $E$.

We now present algorithms to predict the winner of such an election with high probability.
\subsection{Algorithm when MOV is known}
\label{sec:knownmovplurality}
In this section we assume that we know a lower bound $\epsilon N$ on $\text{MOV}(E)$. We present an algorithm that predicts the winner of the election correctly with probability at least $1-\delta$. %The sample complexity of our algorithm is bounded in terms of $\epsilon$ and $\delta$.
\begin{algorithm}[H]
 \caption{}
  \begin{algorithmic}[1]
    %\Require{$(V, (X_v)_{v\in V}, \gamma, (m_v)_{v\in V}, k)$}
    \State{Sample $l_1=\frac{1024}{3\epsilon^2}\log \frac{4}{\delta}$ districts from $D$ uniformly at random with replacement.}
    \State{In each of the sampled districts, sample $l_2=\frac{192}{\epsilon^2}\log \frac{64}{\epsilon}$ votes uniformly at random with replacement and predict their winners using the single-district plurality rule.}
    \Return a candidate that wins in maximum number of sampled districts.
  \end{algorithmic}
\label{alg:alg1}  
\end{algorithm}

Clearly $l_1 = \OO\left(\frac{1}{\epsilon^2}\log \frac{1}{\delta}\right)$ and $l_2 = \OO\left(\frac{1}{\epsilon^2}\log \frac{1}{\epsilon}\right)$. Thus the above algorithm uses $\OO\left(\frac{1}{\epsilon^4}\log \frac{1}{\epsilon}\log \frac{1}{\delta}\right)$ samples.

\begin{lemma}
\label{lem:complexityofsampling}
The sample complexity of \Cref{alg:alg1} is $\OO\left(\frac{1}{\epsilon^4}\log \frac{1}{\epsilon}\log \frac{1}{\delta}\right)$.
\end{lemma}

To analyze the success probability of our \Cref{alg:alg1}, we instead analyze a different algorithm, \Cref{alg:alg2}, whose success probability is immediately seen to be the same as that of \Cref{alg:alg1}.
\begin{algorithm}[H]
 \caption{}
  \begin{algorithmic}[1]
    %\Require{$(V, (X_v)_{v\in V}, \gamma, (m_v)_{v\in V}, k)$}
    \State{From each district, sample $l_2=\frac{192}{\epsilon^2}\log \frac{64}{\epsilon}$ votes uniformly at random with replacement. Let $y_j$ be the candidate that receives the largest number of sampled votes in district $d_j$, $j\in [k]$.}
    %\State{$y_j \leftarrow \text{MAJ}(L'_j)$, $\forall j\in [p]$.}
    \State{Sample $l_1=\frac{1024}{3\epsilon^2}\log \frac{4}{\delta}$ candidates uniformly at random with replacement from the list $(y_1,\ldots,y_k)$. Let the list of sampled candidates be $(z_1,\ldots,z_{l_1})$.}
    \Return $\text{MAJ}(z_1,\ldots,z_{l_1})$.
    
  \end{algorithmic}
\label{alg:alg2}
\end{algorithm}
We show that with probability at least $1-\delta$, the output of \Cref{alg:alg2} is the true winner of $E$.

The crux of our analysis is showing that with high probability the following two statements hold:
\begin{enumerate}
    \item $\text{MAJ}(y_1, \ldots, y_k)$ is the true winner of $E$.
    \item The margin of victory of $(y_1, \ldots, y_k)$, viewed as a single district plurality election, is $\Omega(\epsilon k)$.
\end{enumerate}
Then, applying \Cref{cor:pluralitycor} together with a bookkeeping of the errors incurred in various steps shows that \Cref{alg:alg2} returns the true winner in step 3 with the required probability.

We now attempt to formalise the above notion. Let $c_1,\ldots,c_k$ denote the true winners of the $k$ districts. Let $w=\text{MAJ}(c_1,\ldots,c_k)$ denote the winner of the election and let $w'=\text{SEC-MAJ}(c_1,\ldots,c_k)$. For any candidate $x\in C$, let $f(x)$ denote the number of districts in which $x$ wins. For $x\in C$ and $j\in [k]$, let $g_j(x)$ denote the actual number of votes received by candidate $x$ in district $d_j$.

We first show that $w$ wins in $\Omega(\epsilon k)$ districts more than $w'$.
\begin{lemma}
\label{lem:mov}
$f(w)-f(w')\ge \frac{\epsilon k}{3}$.
\end{lemma}
\begin{proof}
We divide the proof into two cases:

\textbf{Case 1} - $f(w)\ge k/3$: Clearly if the winner is changed from $w$ to $w'$ in $(f(w)-f(w'))$ districts in which $w$ has won, then $w'$ would become the winner of the resulting election. The total population of the least populated $(f(w)-f(w'))$ districts in which $w$ wins is at most $\frac{N}{f(w)}\cdot (f(w)-f(w')) \le \frac{3N}{k}\cdot (f(w)-f(w'))$. Clearly if $w'$ receives all votes in each of these districts, then $w'$ would become the new winner of the election. But since $\text{MOV}(E)\ge \epsilon N$, we must have $\frac{3N}{k}\cdot (f(w)-f(w')) \ge \epsilon N$, implying $f(w)-f(w') \ge \frac{\epsilon k}{3}$, as desired.

\textbf{Case 2} - $f(w) < k/3$: In this case $f(w') \le f(w) < k/3$. Thus there exist more than $k/3$ districts where neither $w$ nor $w'$ has won. If $w'$ is made the winner in $(f(w)-f(w'))$ districts out of these, then clearly $w'$ would become the new winner of the election. Again the total population of the least populated $(f(w)-f(w'))$ such districts is at most $\frac{N}{k/3}\cdot (f(w)-f(w'))$. Hence $\frac{3N}{k}\cdot (f(w)-f(w')) \ge \epsilon N$, implying that $f(w)-f(w') \ge \frac{\epsilon k}{3}$.
\end{proof}

Now we show that for each district $d_j$, with high probability, the numbers of votes secured by $y_j$ and $c_j$ are close.
\begin{lemma}
\label{lem:mov1}
$\text{Pr}\left(g_j(c_j)-g_j(y_j) \le \frac{\epsilon n_j}{4}\right) \ge 1-\frac{\epsilon}{32}$, $\forall j\in [k]$.
\end{lemma}
\begin{proof}
Let $C_j = \{x\in C \mid g_j(c_j)-g_j(x) \le \epsilon n_j/4 \}$. We need to show that $\text{Pr}(y_j \in C_j) \ge 1-\frac{\epsilon}{32}$. Let $X^x_j$ be the random variable denoting the number of sampled votes received by candidate $x$ in district $d_j$. Let $\hat{g}_j(x) = \frac{n_j}{l_2}\cdot X^x_j$. From \Cref{thm:pluralitythm}, we have $\text{Pr}(\forall x\in C, |\hat{g}_j(x)-g_j(x)| \le \epsilon n_j/8) \ge 1-\frac{\epsilon}{32}$. If this holds, then $\hat{g}_j(c_j)$ would be at least $g_j(c_j)-\epsilon n_j/8$ while for any candidate $x\in C\setminus C_j$, $\hat{g}_j(x)$ would be at most $g_j(x)+ \epsilon n_j/8 < g_j(c_j)-\epsilon n_j/8$. Hence $\text{Pr}(y_j \in C_j) \ge 1-\frac{\epsilon}{32}$.
\end{proof}

Let $\EE$ denote the event that the difference $g_j(c_j)-g_j(y_j)$ exceeds $\epsilon n_j/4$ in at most $\epsilon k/16$ districts $d_j$. The next lemma shows that if $k$ is sufficiently large, $\EE$ happens with high probability.
\begin{lemma}
Suppose that $k\ge \frac{96}{\epsilon}\log \frac{4}{\delta}$. Then $\text{Pr}(\EE) \ge 1-\frac{\delta}{2}$.
\end{lemma}
\begin{proof}
Let $Y$ be the random variable denoting the number of districts $d_j$ where $g_j(c_j)-g_j(y_j) > \epsilon n_j/4$. From Lemma \ref{lem:mov1}, it follows that $\mathbb{E}[Y] \le \epsilon k/32$. Using the multiplicative form of Chernoff bound (\Cref{thm:chernoff}) with $\theta = 1$, $\text{Pr}\left(\overline{\EE}\right) = \text{Pr}(Y > \epsilon k/16) \le 2e^{-\frac{\epsilon k}{96}} \le \delta/2$.
\end{proof}
If $\EE$ holds, the list $(y_1,\ldots,y_k)$ can be transformed into another list $(u_1,\ldots,u_k)$ where for each $j\in [k]$, $g_j(c_j)-g_j(u_j) \le \epsilon n_j/4$, by altering at most $\epsilon k/16$ entries. The next lemma lists some properties of the list $(u_1, \ldots, u_k)$.
\begin{lemma}
\label{lem:propertyofu}
Let $(u_1,\ldots,u_k)$ be as defined above. Then
\begin{enumerate}
    \item $\text{MAJ}(u_1,\ldots,u_k) = w$.
    \item Suppose in each district $d_j$, $u_j$ is made the winner by transferring $g_j(c_j)-g_j(u_j)$ votes received by $c_j$ to $u_j$, keeping everything else the same. Let $E'$ denote the resulting election. Then $\text{MOV}(E') \ge 3\epsilon N/4$.
    \item $f(w)-f(\text{SEC-MAJ}(u_1,\ldots,u_k))\ge \epsilon k/4$.
\end{enumerate}
\end{lemma}
\begin{proof}
\begin{enumerate}
    \item Since in district $d_j$, the winner can be changed from $c_j$ to $u_j$ by altering at most $\epsilon n_j/4$ votes, the total number of votes altered to go from $E$ to $E'$ is at most $\sum\limits_{j\in [k]} \epsilon n_j/4 = \epsilon N/4$. Since $\text{MOV}(E)=\epsilon N$, the winner of the election cannot change by altering only $\epsilon N/4$ votes. Thus $w$ must be the winner of the election $E'$ and hence $\text{MAJ}(u_1,\ldots,u_k) = w$.
    \item The number of votes altered to go from $E$ to $E'$ is at most $\epsilon N/4$ as noted in part 1. Hence in order to change the winner of the election, at least further $\epsilon N - \frac{\epsilon N}{4} = \frac{3\epsilon N}{4}$ votes must be altered and therefore $\text{MOV}(E') \ge 3\epsilon N/4$.
    \item Applying Lemma \ref{lem:mov} to the election $E'$, it directly follows that $f(w)-f(\text{SEC-MAJ}$ $(u_1,\ldots,u_k))\ge \frac{(3\epsilon /4) k}{3} =\frac{\epsilon k}{4}$.
\end{enumerate}
\end{proof}

Finally we show that \Cref{alg:alg2} (and hence \Cref{alg:alg1}) returns the true winner with probability at least $1-\delta$.

\begin{lemma}
\label{lem:correctprediction}
\Cref{alg:alg1} predicts the true winner with probability at least $1-\delta$.
\end{lemma}
\begin{proof}
From \Cref{lem:propertyofu}, we have $\text{MAJ}(u_1,\ldots,u_k) = w$ and $f(w)-f(\text{SEC-MAJ}(y_1,\ldots,y_k))\ge \frac{\epsilon k}{4} - \frac{\epsilon k}{16}=\frac{3\epsilon k}{16}$. It follows from Corollary \ref{cor:pluralitycor} that sampling $\frac{3}{(3\epsilon/32)^2}\log \frac{2}{\delta/2} = \frac{1024}{3\epsilon^2}\log \frac{4}{\delta}$ candidates uniformly at random with replacement from the list $(y_1,\ldots,y_k)$ would predict $w$ as the winner with probability at least $1-\frac{\delta}{2}$.

Finally let $\FF$ be the event that \Cref{alg:alg1} does not predict the winner correctly. Then $\text{Pr}(\FF) \le \text{Pr}(\FF | \EE) + \text{Pr}(\overline{\EE})\le \frac{\delta}{2} + \frac{\delta}{2}=\delta$. 
\end{proof}

Combining \Cref{lem:complexityofsampling} and \Cref{lem:correctprediction}, we have the following the result.

\begin{theorem}
There exists an algorithm for \edWD for the plurality rule with sample complexity $\OO\left(\frac{1}{\eps^4}\log \frac{1}{\eps}\log \frac{1}{\delta}\right)$.
\end{theorem}

%%%%%%%%%%%%%%%%%%%%%%%%%%%%%%%%%%%%

%\begin{proof}
%We have $l_1 = \frac{16384}{3\gamma^2}\log \frac{4}{\delta} = O(\frac{1}{\gamma^2}\log \frac{1}{\delta})$ and $l_2 = \frac{192}{\gamma^2}\log \frac{64}{\gamma} = O(\frac{1}{\gamma^2}\log \frac{1}{\gamma})$. Hence the overall sample complexity is given by $l_1 l_2 = O(\frac{1}{\gamma^4}\log \frac{1}{\gamma}\log \frac{1}{\delta})$.
%\end{proof}

\subsection{Optimality}
We now show that the sample complexity of our algorithm is essentially optimal (upto constant factors and logarithmic terms), if we restrict our attention to a special class of algorithms for $(\epsilon,\delta)-${\sc{Winner-Prediction}}.

\begin{theorem}
\label{thm:epsfourlowerbound}
Let $\BB$ be any algorithm that works in the following way (here $l_1$ and $l_2$ depend only on $k,\epsilon,\delta$ and the $n_j$'s):
\begin{enumerate}
    \item Sample $l_1$ districts uniformly at random with replacement from $D$.
    \item Sample $l_2$ votes uniformly at random with replacement from each of the $l_1$ sampled districts and predict their winners using the single-district plurality rule.
    \item Return a candidate that wins in maximum number of sampled districts.
\end{enumerate}
Then for sufficiently small $\epsilon$ and $\delta$, we have $l_1 \ge \frac{1}{64\epsilon^2}\log \frac{1}{8e\sqrt{\pi}\delta}$ and $l_2 \ge \frac{1}{1600\epsilon^2}\log \frac{3}{4e\sqrt{\pi}}$ even when there are 2 candidates A and B and each district has equal population $n = N/k$.
\end{theorem}
\begin{proof}
Similar to the analysis of \Cref{alg:alg1}, we may propose an alternate sampling algorithm $\BB'$, whose probability of predicting the winner is easily seen to be the same as that of $\BB$. $\BB'$ first samples $l_2$ votes uniformly at random with replacement from each district and predicts their winners using the single-district plurality rule. It then samples $l_1$ candidates uniformly at random with replacement from the list of predicted winners (call it $L$) and returns a majority candidate.

%Similar to the analysis of \Cref{alg:alg1}, we may view $\BB$ in the following way: it first samples $l_2$ votes uniformly at random with replacement from each district and predicts their winners using the single-district plurality rule. It then samples $l_1$ candidates uniformly at random with replacement from the list of predicted winners (call it $L$) and returns the majority candidate. This alternate view has the same probability distribution as the original algorithm.

\textbf{Lower bound on $\mathbf{l_1}$:} We provide a reduction from the $(\epsilon,\delta)-${\sc{Winner-Determination}} problem. Consider the following single-district election $E$: there are 2 candidates A and B, and $N'$ voters, out of which $\left (\frac{1}{2}+4\epsilon \right)N'$ vote for A and the remaining vote for B. Clearly then $\text{MOV}(E)=4\epsilon N'$. We create another election $E'$ with the same 2 candidates A and B as follows: for each voter $v$, we create a district $d_v$ consisting of $n$ voters (for some sufficiently large $n$). Let $D^A = \{d_v \mid v \text{ votes for A}\}$ and $D^B = \{d_v \mid v \text{ votes for B}\}$. Let $D=D^A \cup D^B$. Thus $\left |D^A \right| = \left(\frac{1}{2}+4\epsilon \right)N'$, $\left|D^B \right|= \left(\frac{1}{2}-4\epsilon \right)N'$ and $|D|=k=N'$. The total number of voters in the election $E'$ is $N = nk = nN'$. In each $d_v \in D^A$, let A receive $3n/4$ votes and B receive $n/4$ votes while in each $d_v \in D^B$, let B receive all $n$ votes. This completes the description of $E'$. Clearly $\text{MOV}(E') = \frac{n}{4}\cdot 4\epsilon k = \epsilon N$.
    
Now since B receives all votes in each district in $D^B$, the algorithm $\BB'$ would surely predict B as the winner in $|D^B|$ districts. Thus B occurs at least $\left(\frac{1}{2}-4\epsilon \right)k$ times in the list $L$, and therefore A occurs at most $\left(\frac{1}{2}+4\epsilon \right)k$ times. Hence if $l_1 < \frac{1}{64\epsilon^2}\log \frac{1}{8e\sqrt{\pi}\delta}$, we would have an algorithm with sample complexity less than $\frac{1}{4\cdot (4\epsilon)^2}\log \frac{1}{8e\sqrt{\pi}\delta}$ for predicting the winner of $E$ with probability at least $1-\delta$, contradicting \Cref{thm:lowerboundonplurality}.

\textbf{Lower bound on $\mathbf{l_2}$:} Consider the following election $E$: there are 2 candidates A and B and a set $D$ of $k$ districts. Each district has the same population $n$. A wins in $11k/20$ districts (call this set $D^A$) by receiving $\left (\frac{1}{2}+20\epsilon \right)n$ votes in each, and B wins in the remaining districts (set $D^B$) by receiving all the $n$ votes in each. Thus $\text{MOV}(E) = 20\epsilon n \cdot \left(\frac{11k}{20}-\frac{k}{2} \right)= \epsilon N$. 
    
Clearly in each district of $D^B$, B would be predicted as the winner (by $\BB'$). Now consider any $d\in D^A$. If $l_2 < \frac{1}{4\cdot (20\epsilon)^2}\log \frac{1}{8e\sqrt{\pi}\cdot (1/6)}$, then from \Cref{thm:lowerboundonplurality}, the probability that A is predicted as the winner in district $d$ would be at most $1-\frac{1}{6}=\frac{5}{6}$. Let $\YY$ be the random variable denoting the number of districts where A is declared as the winner. Hence $\mathbb{E}[\YY] \le \frac{11k}{20}\cdot \frac{5}{6} = \frac{11k}{24}$. Let $\DD_1$ denote the event that $\YY \le 7k/15$. Using Markov's inequality, $\text{Pr}(\overline{\DD}_1)\le 55/56$. Now let $\WW$ be the random variable denoting the number of sampled districts (in the samples drawn in the second step of $\BB'$) where A is the predicted winner. Let $\DD_2$ denote the event that $\left|\frac{\WW}{l_1} - \frac{\YY}{k} \right|\le \frac{1}{32}$. Since $l_1 \ge \frac{1}{64\epsilon^2}\log \frac{1}{8e\sqrt{\pi}\delta} \ge 3\cdot 32^2\log \frac{2}{\delta}$ (for sufficiently small $\epsilon$), using \Cref{thm:pluralitythm}, $\text{Pr}(\DD_2)\ge 1-\delta$. Finally let $\DD_3$ denote the event that algorithm $\BB$ predicts A as the winner of the election. Note that $\text{Pr}(\DD_3|\DD_1\DD_2)=0$ as conditioned on $\DD_1$ and $\DD_2$, the fraction of sampled districts where A can be the predicted winner is at most $\frac{7}{15}+\frac{1}{32}< \frac{1}{2}$. So $\text{Pr}(\DD_3|\DD_1) = \text{Pr}(\DD_3|\DD_1\overline{\DD}_2)\text{Pr}(\overline{\DD}_2)\le \text{Pr}(\overline{\DD}_2) \le \delta$. Thus $\text{Pr}(\DD_3)\le \text{Pr}(\DD_3|\DD_1)+\text{Pr}(\overline{\DD}_1)\le \delta + \frac{55}{56}< 1-\delta$, a contradiction (assuming $\delta < 1/112$).
\end{proof}

\subsection{Generalization}
Now consider the following setting: suppose the winner of each district is decided using some voting rule $r$ and the overall winner of the election $E$ is a candidate that wins in maximum number of districts. Let $\text{MOV}(E)\ge \epsilon N$. Suppose we wish to predict the winner of such an election with probability at least $1-\delta$. Observe that if, as in the proof of \Cref{lem:mov1}, we can ensure that in each district $d_j$, with high probability (at least $1-\OO(\epsilon)$) $y_j$ can be made the winner by altering at most $\epsilon/4$ fraction of the population of $d_j$, then the rest of the proof would be exactly similar.

Let $\chi_r (m,\epsilon,\delta)$ be the number of samples required so that the predicted winner of a single-district election using using rule $r$ with $n$ voters and $m$ candidates, can be made winner by changing at most $\epsilon n$ votes (note that we need $\chi_r$ to be independent of the population $n$). Then we have the following result.

\iffalse
Observe that if, as in Lemma \ref{lem:mov1}, we can ensure that $\text{Pr}(g_j(c_j)-g_j(y_j)) \le \epsilon n_j/4)\ge 1- \frac{\epsilon}{32}$, $\forall j\in [p]$, using sufficient number of samples, then the rest of the proof would be exactly similar. 

We define the \textit{Margin of Loss} of a candidate $x$ in a district $d$ to be the minimum number of votes to be changed in order to make $x$ the winner in $d$. Suppose voting rule $r$ is used to predict the winner of a single-district election. A candidate $x$ is said to be $\epsilon$-\textit{close} to the true winner if the Margin of Loss of $x$ is at most $\epsilon n$, where $n$ is the population of the district. Let $\chi_r(m,\epsilon,\delta)$ denote the number of samples required so that the predicted winner is $\epsilon$-close to the true winner, with probability at least $1-\delta$. We have the following theorem.
\fi

\begin{theorem}
\label{thm:generalizationofplurality}
There exists an algorithm for \edWD for arbitrary voting rule $r$ with sample complexity $\OO\left(\chi_r(m,\epsilon,\epsilon) \cdot \frac{1}{\epsilon^2}\log \frac{1}{\delta}\right)$.
\end{theorem}

\subsection{Algorithms when MOV is unknown}
\label{sec:algomovunknown}
We now consider two restricted settings of the district-level plurality election. We assume that there are only 2 candidates $A$ and $B$. Wlog assume that A is the true winner of the election. Let $n=N/k$ denote the average population of a district. We assume that no bound on the margin of victory is known to us. We present two algorithms that work even in this setting, whose sample complexity can be bounded in terms of the (unknown) MOV.

\subsubsection{When $n_j \le \kappa n$}
\label{sec:boundedpopulation}
Suppose there exists a parameter $\kappa$ ($\ge 4$) such that the population of each district is at most $\kappa$ times the average population of a district. Let $\text{MOV}(E)=\epsilon N$ which is unknown to the algorithm.% Our algorithm is as follows:

\begin{algorithm}[H]
 \caption{}
  \begin{algorithmic}[1]
    \State{$\gamma \leftarrow \frac{1}{3}$.}
    \State{Sample $l_1=\frac{5\kappa^2}{18\gamma^2}\log \frac{4}{\gamma \delta}$ districts from $D$ uniformly at random with replacement.}
    \State{From each of the sampled districts, sample $l_2=\frac{5\kappa^2}{2\gamma^2}\log \frac{2l_1}{\gamma \delta}$ votes uniformly at random with replacement and predict their winners using the single-district plurality rule.}
    \State{If there exists a candidate that wins in at least $\left(\frac{1}{2}+\frac{3\gamma}{\kappa} \right)l_1$ sampled districts by receiving at least $\left(\frac{1}{2}+\frac{2\gamma}{\kappa}\right)l_2$ sampled votes in each, then declare that candidate as the winner and halt.}
    \State{$\gamma \leftarrow \frac{\gamma}{3}$.}
    \State{\textbf{goto} 2.}
  \end{algorithmic}
\label{alg:alg3}
\end{algorithm}

For estimating the success probability as well as bounding the sample complexity of the above algorithm, we show the following.
\begin{enumerate}
    \item Whenever \Cref{alg:alg3} terminates, it predicts A as the winner with high probability.
    \item As the value of $\gamma$ goes below $(1-\Omega(1))\epsilon$, the probability that \Cref{alg:alg3} does not terminate decreases exponentially with the number of iterations.
\end{enumerate}

The idea is to show that the proportions of votes received by A and B in each sampled district is represented faithfully in the samples drawn in the second step of \Cref{alg:alg3}. Also it can be shown that \Cref{alg:alg3} samples enough districts from the set of districts where A (and B) has won with a ``large" margin of victory. Conditioning on these two events, (1.) follows by showing that whenever \Cref{alg:alg3} terminates, the predicted winner must have won in more than $k/2$ districts and therefore must be the true winner (i.e. A).

For (2.), using the fact that $\text{MOV}(E)=\epsilon N$, it can be shown that A receives at least $\frac{1}{2}+\Omega(\epsilon)$ fraction of votes in at least $\frac{1}{2}+\Omega(\epsilon)$ fraction of districts. Thus, when the value of $\gamma$ goes below $(1-\Omega(1))\epsilon$, \Cref{alg:alg3} terminates with high probability.

We now attempt to formalise the above notions. Let $\tau^A(\gamma)$ (resp. $\tau^B(\gamma)$) denote the fraction of districts where A (resp. B) receives at least $\frac{1}{2}+\frac{\gamma}{\kappa}$ fraction of votes in the election $E$. When $\gamma = (\frac{1}{3})^i$, let $\FF^A_i$ (resp. $\FF^B_i$) denote the event that the fraction of sampled districts where A (resp. B) wins with at least $\frac{1}{2}+\frac{\gamma}{\kappa}$ fraction of votes lies within an additive error of $3\gamma/\kappa$ from $\tau^A(\gamma)$ (resp. $\tau^B(\gamma)$). Let $\FF_i = \FF^A_i \cap \FF^B_i$.

\begin{lemma}
$\text{Pr}\left(\FF_i \right)\ge 1-\frac{\delta}{243^i}$.
\end{lemma}
\begin{proof}
Let $W$ be the random variable denoting the number of sampled districts where A receives at least $\frac{1}{2}+\frac{\gamma}{\kappa}$ fraction of votes. Then $\mathbb{E}[W] = \frac{\tau^A(\gamma)}{k}\cdot l_1$. Using the additive form of Chernoff bound (\Cref{thm:chernoff}) with $\theta = 3\gamma/\kappa$, we have $\text{Pr}\left(\overline{\FF}^A_i \right) = \text{Pr}\left(\left|\frac{W}{l_1}-\frac{\tau^A(\gamma)}{k}\right| \ge \frac{3\gamma}{\kappa}\right)= \text{Pr}\left(\left|W - \frac{\tau^A(\gamma)}{k}\cdot l_1 \right| \ge \frac{3\gamma}{\kappa}\cdot l_1 \right)\le 2e^{-\frac{18\gamma^2}{\kappa^2}\cdot l_1}$. Since $l_1=\frac{5\kappa^2}{18\gamma^2}\log \frac{4}{\gamma \delta} \ge \frac{\kappa^2}{18\gamma^2}\log \frac{4}{\gamma^5 \delta}$, we get $\text{Pr}\left(\overline{\FF}^A_i \right) \le \frac{\gamma^5 \delta}{2} = \frac{\delta}{2\cdot 243^i}$. In a similar way, it follows that $\text{Pr}\left(\overline{\FF}^B_i \right) \le \frac{\delta}{2\cdot 243^i}$. Thus using union bound, we get $\text{Pr}\left(\FF_i \right) \ge 1- \frac{\delta}{243^i}$. 
\end{proof}

Again when $\gamma = (\frac{1}{3})^i$, let $\EE_i$ denote the event that in each of the sampled districts, the fraction of sampled votes received by A lies within an additive error of $\gamma/\kappa$ from the true fraction of votes received by A in that district. Note that if $\EE_i$ holds, then in each district, the sampled fraction of votes received by B also lies within an additive error of $\gamma/\kappa$ from the true fraction of votes received by B in that district.

\begin{lemma}
\label{lem:boundE}
$\text{Pr}(\EE_i)\ge 1-\frac{\delta}{243^i}$.
\end{lemma}
\begin{proof}
Consider any sampled district $d$ and let $\beta_d$ be the true fraction of votes received by A in $d$. Let $Z_d$ be the random variable denoting the number of sampled votes received by A in $d$. Clearly then $\mathbb{E}[Z_d] = \beta_d l_2$. Using the additive form of Chernoff bound (\Cref{thm:chernoff}) with $\theta = \gamma/\kappa$, we have $\text{Pr}\left(\left|\frac{Z_d}{l_2}-\beta_d \right|>\frac{\gamma}{\kappa}\right) = \text{Pr}\left(\left|Z_d-\mathbb{E}[Z_d]\right| > \frac{\gamma}{\kappa}\cdot l_2\right) \le 2e^{-\frac{2\gamma^2}{\kappa^2}\cdot l_2}$. Since $l_2=\frac{5\kappa^2}{2\gamma^2}\log \frac{2l_1}{\gamma \delta} \ge \frac{\kappa^2}{2\gamma^2}\log \frac{2l_1}{\gamma^5 \delta}$, we get $\text{Pr}\left(\left|\frac{Z_d}{l_2}-\beta_d \right|>\frac{\gamma}{\kappa}\right) \le \frac{\gamma^5 \delta}{l_1} = \frac{\delta}{243^i \cdot l_1}$. Now using union bound, $\text{Pr}\left(\overline{\EE}_i \right) = \text{Pr}\left(\exists d, \left|\frac{Z_d}{l_2}-\beta_d \right|>\frac{\gamma}{\kappa} \right) \le \frac{\delta}{243^i}$. 
\end{proof}

Now using the fact that $\text{MOV}(E)=\eps N$, we show the following result.

\begin{lemma}
\label{lem:movlemma}
Assuming $k$ is even and $\epsilon k/\kappa$ is an integer, there exist at least $\left(\frac{1}{2}+\frac{\epsilon}{\kappa}\right)k$ districts where A receives at least $\frac{1}{2}+\frac{\epsilon}{3\kappa}$ fraction of votes.
\end{lemma}
\begin{proof}
Let $D^A$ be the set of districts where A wins and let $|D^A|=\left(\frac{1}{2}+\nu \right)k$. For each $d\in D^A$, let $\text{MOV}(E_d)$ be the minimum number of votes to be changed in district $d$ in order to make B the winner of that district. Wlog let $d_{1},\ldots,d_{\left(\frac{1}{2}+\nu \right)k}$ be the districts of $D^A$ arranged in non-decreasing order of $\text{MOV}(E_d)$. Clearly if B is made the winner in the districts $d_1,\ldots,d_{\nu k + 1}$, then B would become the winner of the election. Since $\text{MOV}(E) = \epsilon N$, we must have $\sum\limits_{j=1}^{\nu k+1} \text{MOV}(E_{d_j}) \ge \epsilon N$. Thus $\text{MOV}(E_{d_{\nu k+1}})\ge \frac{\epsilon N}{\nu k+1}\ge \epsilon n$ (since $\nu \le 1/2$) and therefore $\text{MOV}(E_{d_{j}})\ge \epsilon n$, $\forall \nu k+1 \le j \le \left(\frac{1}{2}+\nu \right)k$.

Now let $I = \{d\in D^A \mid \text{MOV}(E_d)< \epsilon n/2 \}$. Thus we have $\epsilon N \le \sum\limits_{j=1}^{\nu k+1} \text{MOV}(E_{d_{j}}) \le |I|\cdot \frac{\epsilon n}{2} + (\nu k +1- |I|)\cdot (\frac{\kappa n}{2}+1)$ (since the population of each district is at most $\kappa n$, $\text{MOV}(E_d) \le \kappa n/2 + 1$, $\forall d\in D^A$). Algebraic simplification would yield $\nu k-|I| \ge \frac{\left(1-\frac{\nu}{2}\right)\epsilon N - \left(\frac{\kappa n}{2}+1 \right)}{\left(\frac{\kappa n}{2}+1 \right)-\frac{\epsilon n}{2}}> \frac{\frac{3\epsilon N}{4}}{\frac{\kappa n}{2}+1}-1 \ge \frac{\epsilon k}{\kappa}-1$. Since $\nu k - |I|$ is an integer, we must have $\nu k-|I| \ge \epsilon k/\kappa$. Thus in at least $\left(\frac{1}{2}+\frac{\epsilon}{\kappa}\right)k$ districts, $\text{MOV}(E_d)\ge \epsilon n$. Finally since the population of any district is at most $\kappa n$, $\text{MOV}(E_d)\ge \epsilon n$ implies that A receives at least $\frac{1}{2}+\frac{\epsilon}{3\kappa}$ fraction of votes in district $d$. 
\end{proof}

 Now let $S = \{ \eta \mid$ There exist at least $\left(\frac{1}{2}+\frac{\eta}{\kappa} \right)k$ districts where A receives at least $\frac{1}{2}+\frac{\eta}{3\kappa}$ fraction of votes $\}$. Let $\alpha = \max\limits_{\eta \in S} \eta$ and let $D_{\alpha}$ be the set of districts where A receives at least $\frac{1}{2}+\frac{\alpha}{3\kappa}$ fraction of votes. Let $\rho$ be the unique positive integer such that $\left(\frac{1}{3}\right)^{\rho} \le \alpha < \left(\frac{1}{3}\right)^{\rho-1}$. We now argue that \Cref{alg:alg3} terminates with high probability as the value of $\gamma$ goes below $\alpha$.
 
 \begin{lemma}
\label{lem:nonterm}
For $j\ge 2$, the probability that \Cref{alg:alg3} does not terminate when $\gamma = \left(\frac{1}{3}\right)^{\rho+j}$ is at most $\frac{3\delta}{243^{\rho+j}}$.
\end{lemma}
\begin{proof}
By definition, in each district of $D_{\alpha}$, A receives at least $\frac{1}{2}+\frac{\alpha}{3\kappa}$ fraction of votes. Now assuming the event $\EE_{\rho+j}$ holds, if any district from $D_{\alpha}$ is sampled, the fraction of sampled votes received by A in that district would be at least $\frac{1}{2}+ \frac{\alpha}{3\kappa} - \frac{\gamma}{\kappa}\ge \frac{1}{2}+\frac{2\gamma}{\kappa}$ (since $\alpha \ge \left(\frac{1}{3}\right)^{\rho} \ge 9\gamma$). Let $\XX$ be the random variable denoting the number of sampled districts where A wins with at least $\left(\frac{1}{2}+\frac{2\gamma}{\kappa}\right)l_2$ votes. Since $\left|D_{\alpha} \right|\ge \left(\frac{1}{2}+\frac{\alpha}{\kappa}\right)k$, we have $\mathbb{E}[\XX]\ge \left(\frac{1}{2}+\frac{\alpha}{\kappa}\right)l_1$. Using the additive form of Chernoff bound (\Cref{thm:chernoff}) with $\theta = 6\gamma/\kappa$, $\text{Pr}\left(\XX < \left(\frac{1}{2}+\frac{3\gamma}{\kappa}\right)l_1 \mid \EE_{\rho+j}\right) \le \text{Pr}\left(|\XX-\mathbb{E}[\XX]| > \frac{\gamma}{\kappa}\cdot (3^j -3)\cdot l_1 \mid \EE_{\rho+j}\right) \le \text{Pr}\left(|\XX-\mathbb{E}[\XX]| > \frac{6\gamma}{\kappa}\cdot l_1\mid \EE_{\rho+j}\right)\le 2e^{-\frac{72\gamma^2}{\kappa^2}\cdot l_1} \le 2e^{-5\log \frac{1}{\gamma \delta}} \le \frac{2\delta}{243^{\rho+j}}$. The fourth inequality in the above chain of inequalities holds since $l_1 = \frac{5\kappa^2}{18\gamma^2}\log \frac{4}{\gamma \delta} \ge \frac{5\kappa^2}{72\gamma^2}\log \frac{1}{\gamma \delta}$.

Let $F$ be the event that \Cref{alg:alg3} does not terminate when $\gamma = \left(\frac{1}{3}\right)^{\rho+j}$. Then $\text{Pr}(F) \le \text{Pr}(F| \EE_{\rho+j}) + \text{Pr}\left(\overline{\EE}_{\rho+j}\right) \le \frac{3\delta}{243^{\rho+j}}$ (since from Lemma \ref{lem:boundE}, $\text{Pr}(\overline{\EE}_{\rho+j}) \le \frac{\delta}{243^{\rho+j}}$).
\end{proof}

Next we show that \Cref{alg:alg3} predicts A as the winner with high probability, whenever it terminates.

\begin{lemma}
\label{lem:winnerpredict}
If \Cref{alg:alg3} terminates when $\gamma = \left(\frac{1}{3}\right)^i$, it returns A as the winner with probability at least $1- \frac{2\delta}{243^i}$.
\end{lemma}
\begin{proof}
Since \Cref{alg:alg3} terminates, the predicted winner wins in at least $\frac{1}{2}+\frac{3\gamma}{\kappa}$ fraction of sampled districts with at least $\frac{1}{2}+\frac{2\gamma}{\kappa}$ fraction of votes in each district. Now assume that both $\EE_i$ and $\FF_i$ holds true. This happens with probability at least $1-\frac{2\delta}{243^i}$. Since $\EE_i$ holds, the true fraction of votes received by the predicted winner in each of the sampled districts is at least $\frac{1}{2}+\frac{2\gamma}{\kappa}-\frac{\gamma}{\kappa}=\frac{1}{2}+\frac{\gamma}{\kappa}$. Again since $\FF_i$ holds, the true fraction of districts where the predicted winner wins with at least $\frac{1}{2}+\frac{\gamma}{\kappa}$ fraction of votes is more than $\frac{1}{2}+\frac{3\gamma}{\kappa}-\frac{3\gamma}{\kappa}=\frac{1}{2}$. Thus the predicted winner wins in more than half of the districts and therefore must be the true winner.
\end{proof}

Combining the above two results, we show that with probability at least $1-\delta$, \Cref{alg:alg3} returns A as the winner of the election.

\begin{lemma}
\label{lem:probofcorrectprediction}
\Cref{alg:alg3} predicts A as the winner with probability at least $1-\delta$.
\end{lemma}
\begin{proof}
Let $\GG$ denotes the event that \Cref{alg:alg3} predicts A as the winner and $G$ denote the event that \Cref{alg:alg3} terminates when $\gamma = \left(\frac{1}{3}\right)^{\rho+2}$. Then from Lemma \ref{lem:nonterm} and Lemma \ref{lem:winnerpredict}, $\text{Pr}(\GG) \ge \text{Pr}(\GG | G) \text{Pr}(G) \ge \left(1-\frac{2\delta}{243^{\rho+2}}\right)\cdot \left(1-\frac{3\delta}{243^{\rho+2}}\right)\ge 1-\delta$.
\end{proof}

Finally we bound the sample complexity of the algorithm.

\begin{lemma}
\label{lem:expectedsamplecomplexity}
\Cref{alg:alg3} uses at most $\OO \left(\frac{\kappa^4}{\epsilon^4}\log \frac{1}{\epsilon \delta} \log \frac{\kappa}{\epsilon \delta}\right)$ samples in expectation.
\end{lemma}
\begin{proof}
We have $l_1 = \frac{5\kappa^2}{18\gamma^2}\log \frac{4}{\gamma \delta} = \OO \left(\frac{\kappa^2}{\gamma^2}\log \frac{1}{\gamma \delta}\right)$ and $l_2 = \frac{5\kappa^2}{2\gamma^2}\log \frac{2l_1}{\gamma \delta} = \OO \left(\frac{\kappa^2}{\gamma^2}\log \left(\frac{1}{\gamma \delta} \cdot \frac{\kappa^2}{\gamma^2}\log \frac{1}{\gamma \delta}\right)\right) = \OO\left(\frac{\kappa^2}{\gamma^2}\log \frac{\kappa}{\gamma \delta}\right)$. Thus for a particular value of $\gamma$, \Cref{alg:alg3} collects a total of $l_1l_2 = \OO \left(\frac{\kappa^4}{\gamma^4}\log \frac{1}{\gamma \delta}\log \frac{\kappa}{\gamma \delta}\right)$ samples. When $\gamma = \left(\frac{1}{3}\right)^i$, this equals $\OO\left(\kappa^4 81^i \left(i+\log 
\frac{1}{\delta}\right)\left(i+\log \frac{\kappa}{\delta}\right)\right)$. Let $q(j)$ denote the total number of samples collected by \Cref{alg:alg3} if it halts when $\gamma = \left(\frac{1}{3}\right)^{\rho+j}$. Thus $q(j)=\sum\limits_{i=1}^{\rho+j} \OO\left(\kappa^4 81^i \left(i+\log 
\frac{1}{\delta}\right)\left(i+\log \frac{\kappa}{\delta}\right)\right) \le \OO\left(\frac{\kappa^4}{\alpha^4} \cdot 81^j \left(\rho+j+\log \frac{1}{\delta}\right)\left(\rho+j+\log \frac{\kappa}{\delta}\right)\right)$. For $j\ge 3$, let $H_j$ denote the event that \Cref{alg:alg3} terminates when $\gamma = \left(\frac{1}{3}\right)^{\rho+j}$ and let $H_2$ denote the event that \Cref{alg:alg3} terminates for some $\gamma$ in the set $\{\frac{1}{3},\ldots,\left(\frac{1}{3}\right)^{\rho+2} \}$. Then from Lemma \ref{lem:nonterm}, for $j\ge 3$, $\text{Pr}(H_j) \le \text{Pr (\Cref{alg:alg3} does not terminate when } \gamma = \left(\frac{1}{3}\right)^{\rho+j-1}) \le \frac{3\delta}{243^{\rho+j-1}}$, and $\text{Pr}(H_2)\le 1$. Thus the expected sample complexity is given by $\sum\limits_{j=2}^{\infty} q(j)\text{Pr}(H_j) \le q(2) + \sum\limits_{j=3}^{\infty} \OO\left(\frac{\kappa^4}{\alpha^4} \cdot 81^j \left(\rho+j+\log \frac{1}{\delta}\right)\left(\rho+j+\log \frac{\kappa}{\delta}\right) \cdot \frac{3\delta}{243^{\rho+j-1}}\right)$. Now $q(2)=\OO \left(\frac{\kappa^4}{\alpha^4} \left(\rho+\log \frac{1}{\delta}\right)\left(\rho+\log \frac{\kappa}{\delta}\right)\right)$ while the second term in the sum is at most $\OO\left(\frac{\kappa^4}{\alpha^4} \left(\rho+\log \frac{1}{\delta}\right)\left(\rho+\log \frac{\kappa}{\delta}\right) \cdot \frac{\delta}{243^{\rho}}\right)$. Hence the overall expected sample complexity is bounded by $\OO \left(\frac{\kappa^4}{\alpha^4}\log \frac{1}{\alpha \delta} \log \frac{\kappa}{\alpha \delta}\right)$ (as $\rho=\OO\left(\log \frac{1}{\alpha}\right)$). Since $\alpha \ge \epsilon$, the expected sample complexity is at most $\OO\left(\frac{\kappa^4}{\epsilon^4}\log \frac{1}{\epsilon \delta} \log \frac{\kappa}{\epsilon \delta}\right)$.
\end{proof}

Combining \Cref{lem:probofcorrectprediction} and \Cref{lem:expectedsamplecomplexity}, we get the following result.

\begin{theorem}
There exists an algorithm for \dWD for the plurality rule with expected sample complexity $\OO\left(\frac{\kappa^4}{\epsilon^4}\log \frac{1}{\epsilon \delta} \log \frac{\kappa}{\epsilon \delta}\right)$ when there are 2 candidates and the population of each district is at most $\kappa$ times the average population of a district.
\end{theorem}

%, whose proof is given in \Cref{sec:sectwo}.

%The above lemma shows that, conditioned on the fact that the algorithm terminates, the true winner is predicted with high probability. This conditioning can be removed with the help of Lemma \ref{lem:nonterm}, where we showed that the algorithm must terminate with high probability when the value of $\epsilon$ goes sufficiently below $\alpha$.

%The sample complexity of \Cref{alg:alg1} was $O(\frac{1}{\epsilon^4}\log \frac{1}{\epsilon} \log \frac{1}{\delta})$ (Theorem \ref{thm:samplecomplexity}). Hence we see that, by a small amplification in the logarithmic terms in the sample complexity, we are able to predict the winner even when the margin of victory is not known to the algorithm. 

%The only dissatisfying assumption we made was the fact that there are only 2 candidates and the population of each district is bounded. It would be a nice progress if these assumptions could be lifted in future works.

\subsubsection{When $n_j$ is arbitrary}
\label{sec:winnerpredictwitharbitrarypop}
We now consider the case when the populations of the districts can be arbitrary. Our algorithm is as follows:

\begin{algorithm}[H]
 \caption{}
  \begin{algorithmic}[1]
    \State{$\gamma \leftarrow \frac{1}{2}$.}
    \State{Sample $l_1=\frac{175}{2\gamma^2}\log \frac{4}{\gamma \delta}$ districts from $D$ uniformly at random with replacement.}
    \State{From each of the sampled districts, sample $l_2=\frac{57344}{9\gamma^4}\log \frac{2l_1}{\gamma \delta}$ votes uniformly at random with replacement and predict their winners using the single-district plurality rule.}
    \State{If there exists a candidate that wins in at least $\frac{1}{2}+\frac{\gamma}{5}$ fraction of the sampled districts with at least $\frac{1}{2}+\frac{5\gamma^2}{128}$ fraction of sampled votes in each, then declare that candidate as the winner and halt.}
    \State{$\gamma \leftarrow \frac{\gamma}{2}$.}
    \State{\textbf{goto} 2.}
  \end{algorithmic}
\label{alg:alg4}
\end{algorithm}

The main difference in the analysis is unlike in the previous case, where we were able to show that A receives at least $\frac{1}{2}+\Omega(\epsilon)$ fraction of votes in at least $\frac{1}{2}+\Omega(\epsilon)$ fraction of districts (\Cref{lem:movlemma}), the current setting enables us to show that A receives at least $\frac{1}{2}+\Omega(\epsilon)$ fraction of votes only in  $\frac{1}{2}+\Omega(\epsilon^2)$ fraction of districts. As in the proof of Lemma \ref{lem:movlemma}, let $D^A$ be the set of districts where A wins and let $\left|D^A \right|=\left(\frac{1}{2}+\nu \right)k$ (assume that $k$ is even, so that $\nu k$ is an integer). For $d\in D^A$, $\text{MOV}(E_d)$ denotes the minimum number of votes to be changed in district $d$ in order to make B the winner of that district. Wlog let $d_{1},\ldots,d_{\left(\frac{1}{2}+\nu \right)k}$ be the districts of $D_A$ arranged in non-decreasing order of $\text{MOV}(E_d)$. We first show the following two results.

\begin{lemma}
\label{lem:lowernad}
$\epsilon n \le \text{MOV}\left(d_{\nu k +1} \right)\le 4n$.
\end{lemma}
\begin{proof}
Since $\text{MOV}(E) = \epsilon N$, $\sum\limits_{j=1}^{\nu k + 1} \text{MOV}(E_{d_{j}}) \ge \epsilon N$. Thus $\text{MOV}(E_{d_{\nu k}})\ge \frac{\epsilon N}{\nu k +1}\ge \epsilon n$ (since $\nu \le 1/2$), which proves the first inequality.

For the second inequality, assume to the contrary that $\text{MOV}(E_{d_{\nu k+1}}) > 4n$. Thus $\text{MOV}(E_{d_{j}})>4n$, $\forall \nu k+1 \le j \le (\frac{1}{2}+\nu)k$. This implies that the population of each of the districts $d_{\nu k+1},\ldots,d_{\left(\frac{1}{2}+\nu \right)k}$ is at least $4n$. Hence the total population of these $k/2$ districts is at least $4n \cdot \frac{k}{2} = 2N > N$, a contradiction. 
\end{proof}

\begin{lemma}
\label{lem:lowerbound}
Let $\II = \{d\in D^A \mid \text{MOV}(E_d)\ge \epsilon n/2\}$. Then $|\II| \ge \left(\frac{1}{2}+\frac{3\epsilon}{16}\right)k$ (assuming $3\epsilon k/16$ is an integer).
\end{lemma}
\begin{proof}
From Lemma \ref{lem:lowernad}, we have $\text{MOV}(E_{d_{\nu k +1}}) \ge \epsilon n$, implying that $\text{MOV}(E_{d_j}) \ge \epsilon n> \epsilon n/2$, $\forall \nu k+1 \le j \le \left(\frac{1}{2}+\nu \right)k$. Also $\text{MOV}(E_{d_{\nu k +1}}) \le 4n$ implies that $\text{MOV}(E_{d_j}) \le 4n$, $\forall 1 \le j \le \nu k+1$. Let $J=D^A \setminus \II$. Then $\epsilon N \le \sum\limits_{j=1}^{\nu k +1} \text{MOV}(E_{d_j}) \le |J|\cdot \frac{\epsilon n}{2} + (\nu k+1 - |J|)\cdot 4n$. Simplifying, we get $\nu k - |J| \ge \frac{\left(1-\frac{\nu}{2}\right)\epsilon N - 4n}{4n - \frac{\epsilon n}{2}} > \frac{\frac{3\epsilon N}{4}-4n}{4n} = \frac{3\epsilon k}{16} -1$. Since $\nu k - |J|$ is an integer, we have $\nu k - |J| \ge 3\epsilon k/16$. Hence $|\II| \ge \left(\frac{1}{2}+\frac{3\epsilon}{16}\right)k$.   
\end{proof}

Using the above two lemmas, we show a lower bound of $\frac{1}{2}+\Omega(\eps)$ on the fraction of districts where A wins with at least $\frac{1}{2}+\Omega(\eps^2)$ fraction of votes.

\begin{lemma}
There exist at least $\left(\frac{1}{2}+\frac{\epsilon}{10} \right)k$ districts where A receives at least $\frac{1}{2}+\frac{\epsilon^2}{64}$ fraction of votes.
\end{lemma}
\begin{proof}
Let $\II$ be defined as in Lemma \ref{lem:lowerbound} and let $\JJ \subseteq \II$ be the set of districts $d_j$ such that $\text{MOV}(E_{d_j}) \le \epsilon^2 n_j/64$. Thus in each of the $|\JJ|$ districts, we have $\frac{\epsilon n}{2}\le \frac{\epsilon^2 n_j}{64}$, implying that $n_j \ge 32n/\epsilon$. Also since $\text{MOV}(E_d)\ge \epsilon n/2$, $\forall d\in \II\setminus \JJ$, the population of each district in $\II \setminus \JJ$ is at least $\epsilon n -1$. Since the total population of all districts is $N$, we must have $|\JJ|\cdot \frac{32n}{\epsilon} + \left( \left(\frac{1}{2}+\frac{3\epsilon}{16}\right)k-|\JJ|\right)\cdot (\epsilon n-1) \le N$, which on simplification yields $|\JJ| \le \frac{n-(\frac{1}{2}+\frac{3\epsilon}{16})\cdot (\epsilon n-1)}{32-\epsilon^2}\cdot \frac{\epsilon k}{n} \le \frac{\epsilon k}{31}$. Therefore $|D^A \setminus \JJ|\ge |\II \setminus \JJ| \ge \left(\frac{1}{2}+\frac{3\epsilon}{16} - \frac{\epsilon}{31}\right)k \ge \left(\frac{1}{2}+\frac{\epsilon}{10}\right)k$. The desired result follows since $\text{MOV}(E_{d_j}) > \epsilon^2 n_j/64$ implies A receives at least $\frac{1}{2}+\frac{\epsilon^2}{64}$ fraction of votes in district $d_j$. 
\end{proof}

Now let $\tau'^A(\gamma)$ (resp. $\tau'^B(\gamma)$) denote the fraction of districts where A (resp. B) receives at least $\frac{1}{2}+\frac{\gamma^2}{64}$ fraction of votes. When $\gamma = (\frac{1}{2})^i$, let $\FF'^A_i$ (resp. $\FF'^B_i$) denote the event that the fraction of sampled districts where A (resp. B) wins with at least $\frac{1}{2}+\frac{\gamma^2}{64}$ fraction of votes lies within an additive error of $\gamma/5$ from $\tau'^A(\gamma)$ (resp. $\tau'^B(\gamma)$). Let $\FF'_i = \FF'^A_i \cap \FF'^B_i$. Similarly when $\gamma = (\frac{1}{2})^i$, let $\EE'_i$ denote the event that in each of the sampled districts, the fraction of sampled votes received by A lies within an additive error of $3\gamma^2/128$ from the true fraction of votes received by A in that district. 

Let $S' = \{ \eta' \mid$ There exist at least  $\left(\frac{1}{2}+\frac{\eta'}{10}\right)k$ districts where A receives at least $\frac{1}{2}+\frac{\eta'^2}{64}$ fraction of votes $\}$. Let $\alpha' = \max\limits_{\eta' \in S'} \eta'$. Let $\rho'$ be the unique positive integer such that $\left(\frac{1}{2}\right)^{\rho'} \le \alpha < \left(\frac{1}{2}\right)^{\rho'-1}$.

We state a series of lemmas whose proofs are analogous to the corresponding lemmas in the analysis of \Cref{alg:alg3} (see \Cref{sec:boundedpopulation}).

\begin{lemma}
$\text{Pr}(\FF'_i) \ge 1-\frac{\delta}{128^i}$.
\end{lemma}

\begin{lemma}
$\text{Pr}(\EE'_i) \ge 1-\frac{\delta}{128^i}$.
\end{lemma}

\begin{lemma}
The probability that \Cref{alg:alg4} does not terminate when $\gamma = \left(\frac{1}{2}\right)^{\rho'+j}$, $j\ge 2$, is at most $\frac{3\delta}{128^{\rho'+j}}$.
\end{lemma}

\begin{lemma}
If \Cref{alg:alg4} terminates when $\gamma = \left(\frac{1}{2}\right)^i$, then it predicts A as the winner with probability at least $1-\frac{2\delta}{128^i}$.
\end{lemma}

Using the above results, similar to the proofs of \Cref{lem:probofcorrectprediction} and \Cref{lem:expectedsamplecomplexity}, it follows that \Cref{alg:alg4} predicts the true winner with probability at least $1-\delta$ and uses at most $\OO \left(\frac{1}{\alpha'^6}\log^2 \frac{1}{\alpha' \delta}\right) \le \OO \left(\frac{1}{\epsilon^6}\log^2 \frac{1}{\epsilon \delta}\right)$ samples in expectation. We thus have the following result.

\begin{theorem}
There exists an algorithm for \dWD for the plurality rule with expected sample complexity $\OO \left(\frac{1}{\epsilon^6}\log^2 \frac{1}{\epsilon \delta}\right)$, when there are 2 candidates.
\end{theorem}

%See \Cref{sec:sectwo} for a proof of the following result.
\iffalse
\begin{lemma}
\label{lem:squaredmov}
There exist at least $(\frac{1}{2}+\frac{\epsilon}{10})p$ districts where A receives at least $\frac{1}{2}+\frac{\epsilon^2}{64}$ fraction of votes.
\end{lemma}

The rest of the analysis is analogous to the analysis of \Cref{alg:alg3}. See \Cref{sec:proofofthm17} for details of the proof of the following theorem.
\fi

%Lifting the assumption of bounded population of each district to arbitrary populations resulted in the sample complexity jumping from $\Tilde{O}(\frac{1}{\gamma^4})$ to $\Tilde{O}(\frac{1}{\gamma^6})$. It would be a nice advancement to design an algorithm with $\Tilde{O}(\frac{1}{\gamma^4})$ (or even $\Tilde{O}(\frac{1}{\gamma^2})$!) sample complexity for the arbitrary case.
\section{Winner Prediction for Median Rule}
\label{sec:median}
We now turn our attention to another popular voting rule - the median rule. Here there is a \textit{harmonious order} $\RR = (c_1,\ldots,c_m)$ over the candidates. For each candidate $x\in C$, let $g(x)$ denote the number of votes where $x$ is most preferred. Then the winner of the election is the candidate $c_t$ such that (i) $\sum_{i=1}^{t} g(c_i) \ge N/2$, and (ii) $\sum_{i=1}^{t-1} g(c_i) < N/2$.

Like the plurality rule, the median rule is also an example of a top-ranked voting rule. We first restrict our attention only to the single-district case. We assume that a lower bound of $\epsilon N$ is known on the Margin of Victory of the election.

\subsection{Algorithm when Harmonious Order is Known}
\label{sec:knownR}
Let us first consider the setting where the harmonious order $\RR$ is known to the algorithm. Our algorithm is as follows:

\begin{algorithm}[H]
 \caption{}
  \begin{algorithmic}[1]
    %\Require{$(V, (X_v)_{v\in V}, \gamma, (m_v)_{v\in V}, k)$}
    \State{Sample $l=\frac{1}{2\epsilon^2}\log \frac{4}{\delta}$ votes uniformly at random with replacement. For $x\in C$, Let $h(x)$ denote the number of sampled votes received by candidate $x$.}
    \State{Let $c_s$ be the candidate such that $\sum_{i=1}^{s} h(c_i) \ge l/2$ and $\sum_{i=1}^{s-1} h(c_i) < l/2$.}
    \Return{$c_s$}
  \end{algorithmic}
\label{alg:alg6}
\end{algorithm}

%We show the following result, whose complete proof can be found in \Cref{sec:proofofthm18}.
Let $c_t$ be the true winner of the election. Using the fact that $\text{MOV}(E)\ge \epsilon N$, we show that there must exist a gap of at least $\epsilon N$ between $\sum_{i=1}^{t-1} g(c_i)$ and $N/2$, and between $N/2$ and $\sum_{i=1}^{t} g(c_i)$.

\begin{lemma}
\label{lem:movgap}
$\sum\limits_{i=1}^{t-1} g(c_i) \le \frac{N}{2}-\epsilon N$.
\end{lemma}
\begin{proof}
Clearly $c_{t-1}$ can be made the winner of the election by transferring $\frac{N}{2}-\sum\limits_{i=1}^{t-1} g(c_i)$ votes received by some candidate(s) in the set $\{c_t,\ldots,c_m\}$, to $c_{t-1}$. Thus $\frac{N}{2}-\sum\limits_{i=1}^{t-1} g(c_i) \ge \epsilon N$, implying that $\sum\limits_{i=1}^{t-1} g(c_i) \le \frac{N}{2}-\epsilon N$.
\end{proof}

\begin{lemma}
\label{lem:movmediangap}
$\sum\limits_{i=1}^{t} g(c_i) \ge \frac{N}{2}+\epsilon N$.
\end{lemma}
\begin{proof}
$c_{t+1}$ can be made the winner by transferring $\sum\limits_{i=1}^{t} g(c_i) - \frac{N}{2}$ votes received by some candidate(s) in the set $\{c_1,\ldots,c_t\}$, to $c_{t+1}$. Thus $\sum\limits_{i=1}^{t} g(c_i) - \frac{N}{2} \ge \epsilon N$, implying $\sum\limits_{i=1}^{t} g(c_i) \ge \frac{N}{2}+\epsilon N$. \end{proof}

Now let $\HH_1$ denote the event that the number of sampled votes received by the set of candidates $\{c_1,\ldots,c_{t-1}\}$, is less than $l/2$ and $\HH_2$ denote the event that the number of sampled votes received by the candidates $\{c_1,\ldots,c_t\}$, is at least $l/2$. Clearly if $\HH_1\cap \HH_2$ holds, then \Cref{alg:alg6} would predict $c_t$ to be the winner of the election. Next we show that $\HH_1$ and $\HH_2$ each hold with probability at least $1-\frac{\delta}{2}$.

\begin{lemma}
\label{lem:probe1}
$\text{Pr}(\HH_1) \ge 1-\frac{\delta}{2}$.
\end{lemma}
\begin{proof}
Let $Z$ be the random variable denoting the number of sampled votes received by $\{c_1,\ldots,c_{t-1}\}$. Clearly $\mathbb{E}[Z]=\frac{\sum\limits_{i=1}^{t-1} g(c_i)}{N} \cdot l \le \left(\frac{1}{2}-\epsilon \right)l$ (from \Cref{lem:movgap}). Using the additive form of Chernoff bound (\Cref{thm:chernoff}) with $\theta = \epsilon$, $\text{Pr}(Z \ge l/2) \le \text{Pr}(|Z-\mathbb{E}[Z]|\ge \epsilon l) \le 2e^{-2\epsilon^2 l} = \delta/2$. 
\end{proof}

\begin{lemma}
\label{lem:probe2}
$\text{Pr}(\HH_2) \ge 1-\frac{\delta}{2}$.
\end{lemma}
\begin{proof}
Let $\ZZ$ be the random variable denoting the number of votes received by the set of candidates $\{c_1,\ldots,c_{t}\}$. Then $\mathbb{E}[\ZZ] = \frac{\sum\limits_{i=1}^{t} g(c_i)}{N} \cdot l\ge \left(\frac{1}{2}+\epsilon \right)l$. Again applying the additive form of the Chernoff bound with $\theta = \epsilon$, $\text{Pr}(\ZZ<l/2) \le \text{Pr}(|\ZZ-\mathbb{E}[\ZZ]|\ge \epsilon l) \le 2e^{-2\epsilon^2 l} = \delta/2$.
\end{proof}

Thus using union bound, the probability that at least one of $\HH_1$ or $\HH_2$ does not hold is at most $\delta$. Hence $\text{Pr}(\HH_1 \cap \HH_2) \ge 1-\delta$ and therefore \Cref{alg:alg6} returns the true winner with probability at least $1-\delta$. The sample complexity of \Cref{alg:alg6} is easily seen to be $\OO\left(\frac{1}{\eps^2}\log \frac{1}{\delta}\right)$. Hence we have the following result.

\begin{theorem}
\label{thm:winnerdetmedianrule}
There exists an algorithm for $(\epsilon,\delta)-${\sc{Winner-Determination}} for the median rule with sample complexity $\OO\left(\frac{1}{\eps^2}\log \frac{1}{\delta}\right)$ when the Harmonious Order is known.
\end{theorem}

%We thus have the following theorem.
\iffalse
\begin{theorem}
\label{thm:medianrule}
There exists an algorithm with sample complexity $O(\frac{1}{\gamma^2}\log \frac{1}{\delta})$, that correctly predicts the winner of an election $E$ where the median rule is used and the harmonious order $\RR$ is known, with probability at least $1-\delta$, where $\text{MOV}(E)=\gamma N$.
\end{theorem}
\fi

\subsection{Algorithm when Harmonious Order is not known}
\label{sec:algowhenrnotknown}
Now we consider the more realistic setting where the harmonious order $\RR$ is not known to the algorithm. For this case, we make the assumption that the preference (we shall use the terms vote and preference interchangeably) of each voter is \textit{single-peaked} with respect to some order $\RR' = (c'_1,\ldots,c'_m)$; this means that for each vote $v^i\in \LL(C)$, $i\in [N]$, there exists a candidate $c'^i_s$ such that in the preference order $v^i$, we have $c'^i_s \succ c'^i_{s-1} \succ \ldots \succ c'^i_{1}$ and $c'^i_s \succ c'^i_{s+1} \succ \ldots \succ c'^i_{m}$. We say that the preference $v^i$ is single-peaked with respect to $c'^i_s$.

As before, the winner of the election is the candidate $c_t$ such that $\sum_{i=1}^t g(c_i) \ge N/2$ and $\sum_{i=1}^{t-1} g(c_i) < N/2$ (where $\RR=(c_1,\ldots,c_m)$ is the harmonious order). The following result has been known to folklore.

\begin{lemma}
\label{thm:medianthm}
The winner of an election where the median rule is used, when preferences are single-peaked with respect to some order, is the Condorcet winner of the election if the number of voters is odd.
\end{lemma}

Our algorithm is as follows:

\begin{algorithm}[H]
 \caption{}
  \begin{algorithmic}[1]
    %\Require{$(V, (X_v)_{v\in V}, \gamma, (m_v)_{v\in V}, k)$}
    \State{Sample $l=\frac{1}{2\epsilon^2}\log \frac{4}{\delta}$ votes uniformly at random with replacement.}
    \State{For any $x,y\in C$, let $h(x,y)$ denote the number of sampled votes where $x$ is preferred over $y$.}
    \State{For $x,y\in C$, let $h'(x,y) = h(x,y)-h(y,x)$.}
    \Return the candidate $x$ such that $h'(x,y)>0$, $\forall y\neq x$.
  \end{algorithmic}
\label{alg:alg7}
\end{algorithm}

%We show the following result. 
\Cref{alg:alg7} outputs the Condorcet winner out of a sample of $l$ votes. Since the preferences are single-peaked, it follows from \Cref{thm:medianthm} that \Cref{alg:alg7} in fact predicts the winner resulting by applying the median rule with respect to the order $\RR$, on the $l$ samples. It has already been proven in the previous subsection that that this would predict the winner correctly with probability at least $1-\delta$.

\begin{theorem}
\label{thm:winnerdetmedianrule1}
There exists an algorithm for $(\epsilon,\delta)-${\sc{Winner-Determination}} for the median rule with sample complexity $\OO\left(\frac{1}{\eps^2}\log \frac{1}{\delta}\right)$ when the Harmonious Order is not known.
\end{theorem}

\subsection{Optimality}
It is clear that the median rule reduces to the plurality rule when there are 2 candidates. \Cref{thm:lowerboundonplurality} gives a lower bound of $\Omega\left(\frac{1}{\epsilon^2}\log \frac{1}{\delta}\right)$ on the sample complexity for $(\epsilon,\delta)-${\sc{Winner-Determination}} for any voting rule that reduces to the single-district plurality rule for 2 candidates. Thus the sample complexities of \Cref{alg:alg6} and \Cref{alg:alg7} are optimal upto constant factors.

\subsection{The multiple districts case}
We now consider the case when the $N$ voters are arranged into $k$ districts. The winner of each district is decided by applying the median rule. The Harmonious orders in the districts may or may not be the same and may or may not be known to us. If the Harmonious order is unknown in a district, we make the assumption that the preference of each voter in that district is single-peaked with respect to some order $\RR'$. The overall winner of the election is a candidate that wins in maximum number of districts. 

It follows as a corollary of \Cref{thm:winnerdetmedianrule} and \Cref{thm:winnerdetmedianrule1} that when $r$ is the median rule, $\chi_r(m,\epsilon,\delta) = \OO \left(\frac{1}{\epsilon^2}\log \frac{1}{\delta}\right)$. Thus assuming a lower bound of $\epsilon N$ on the Margin of Victory of the election $E$, we get the following result using \Cref{thm:generalizationofplurality}.

\begin{corollary}
\label{cor:mediancorollary}
There exists an algorithm with sample complexity $\OO\left(\frac{1}{\epsilon^4}\log \frac{1}{\epsilon}\log \frac{1}{\delta}\right)$ for \edWD for the median rule.
\end{corollary}
\section{Winner Prediction with imperfect Samples}
Till now, we assumed that we could get uniform random samples from the population. However this might not always be the case. We now present algorithms for $(\epsilon,\delta,\gamma)-${\sc{Winner-Determination}} and $(\epsilon,\delta,\gamma)-${\sc{Winner-Prediction}}.

\subsection{Algorithm for single-district election}
Suppose plurality rule is used to determine the winner. Our algorithm is as follows:

\begin{algorithm}[H]
 \caption{}
  \begin{algorithmic}[1]
  
    \State{Sample $\frac{3}{(\eps - \gamma)^2}\log \frac{2}{\delta}$ votes from the distribution $U$ without replacement.}
    \Return{ a candidate that receives the largest number of sampled votes.}
    
  \end{algorithmic}
\label{alg:alg20}
\end{algorithm}

We continue to use the same notations as introduced in \Cref{sec:plurality}. For $x\in C$, let $\SS(x)$ be the set of candidates that vote for $x$. Let $w = \arg\max_{x\in C} |\SS(x)|$ be the winner of the election and $w' = \arg\max_{x\in C \setminus \{w\}} |\SS(x)|$ be a candidate receiving the second largest number of votes. Since $\text{MOV}(E)\ge \eps N$, we must have $|\SS(w)| - |\SS(w')| \ge 2\eps N - 1$. Since $d_{\text{TV}} (U,V)\le \gamma$, we immediately have the following result.

\begin{lemma}
$\sum_{i\in \SS(w)} p_i \ge \frac{|\SS(w)|}{N} - \gamma$, $\sum_{i\in \SS(w')} p_i \le \frac{|\SS(w')|}{N} + \gamma$.
\end{lemma}

Hence we get $\sum_{i\in \SS(w)} p_i - \sum_{i\in \SS(w')} p_i \ge \frac{|\SS(w)|-|\SS(w')|}{N}-2\gamma \ge 2(\eps - \gamma)-\frac{1}{N}$. Thus the margin of victory of the election \textit{with respect to} the distribution $U$ is at least $(\eps - \gamma)N$. Hence using \Cref{cor:pluralitycor}, $\frac{3}{(\eps - \gamma)^2}\log \frac{2}{\delta}$ samples are enough to predict the winner correctly with probability at least $1-\delta$. Hence we have the following result.

\begin{theorem}
\label{lem:pluralitywithnoise}
There exists an algorithm for $(\epsilon,\delta,\gamma)-${\sc{Winner-Determination}} for the plurality rule with sample complexity $\OO\left(\frac{1}{(\eps - \gamma)^2}\log \frac{1}{\delta}\right)$.
\end{theorem}

\subsection{Algorithm for the multiple-districts case}
Our algorithm and its analysis is very similar to that of \Cref{alg:alg1}.

\begin{algorithm}[H]
 \caption{}
  \begin{algorithmic}[1]
    %\Require{$(V, (X_v)_{v\in V}, \gamma, (m_v)_{v\in V}, k)$}
    \State{Sample $l_1=\frac{3072}{(3\epsilon-32\gamma)^2}\log \frac{4}{\delta}$ districts from $D$ uniformly at random with replacement.}
    \State{In each of the sampled districts, sample $l_2=\frac{192}{(\epsilon-\gamma)^2}\log \frac{64}{\epsilon}$ votes uniformly at random with replacement and predict their winners using the single-district plurality rule.}
    \Return a candidate that wins in maximum number of sampled districts.
  \end{algorithmic}
\label{alg:alg21}  
\end{algorithm}

The sample complexity of the above algorithm is easily seen to be $\OO\left(\frac{1}{\eps^4}\log \frac{1}{\eps}\log \frac{1}{\delta}\right)$. For analysing the success probability, we may again propose an alternate sampling algorithm whose probability of predicting the winner is same as that of \Cref{alg:alg21}.

\begin{algorithm}[H]
 \caption{}
  \begin{algorithmic}[1]
    %\Require{$(V, (X_v)_{v\in V}, \gamma, (m_v)_{v\in V}, k)$}
    \State{From each district, sample $l_2=\frac{192}{(\epsilon-\gamma)^2}\log \frac{64}{\epsilon}$ votes uniformly at random with replacement. Let $y_j$ be the candidate that receives the largest number of sampled votes in district $d_j$, $j\in [k]$.}
    %\State{$y_j \leftarrow \text{MAJ}(L'_j)$, $\forall j\in [p]$.}
    \State{Sample $l_1=\frac{3072}{(3\epsilon-32\gamma)^2}\log \frac{4}{\delta}$ candidates uniformly at random with replacement from the list $(y_1,\ldots,y_k)$. Let the list of sampled candidates be $(z_1,\ldots,z_{l_1})$.}
    \Return $\text{MAJ}(z_1,\ldots,z_{l_1})$.
    
  \end{algorithmic}
\label{alg:alg22}
\end{algorithm}

Now it can be easily verified that \Cref{lem:mov1} continues to hold, i.e. $\text{Pr}\left(g_j(c_j)-g_j(y_j) \le \eps n_j/4 \right) \ge 1-\frac{\eps}{32}$. Thus for sufficiently large $k$, with probability at least $1-\frac{\delta}{2}$, the difference $g_j(c_j)-g_j(y_j)$ would exceed $\eps n_j/4$ in at most $\eps k/16$ districts $d_j$. Conditioning on this, the list $(y_1,\ldots,y_k)$ can be transformed into a list $(u_1,\ldots,u_k)$ such that $g_j(c_j)-g_j(y_j) \le \eps n_j/4$, $\forall j\in [k]$ by altering at most $\eps k/16$ entries. From \Cref{lem:propertyofu}, $\text{MAJ}(u_1,\ldots,u_k)=w$ and $f(w)-f(\text{SEC-MAJ}(u_1,\ldots,u_k))\ge \eps k/4$. Thus $f(w)-f(\text{SEC-MAJ}(y_1,\ldots,y_k))\ge \frac{\eps k}{4}-\frac{\eps k}{16}=\frac{3\eps k}{16}$. Hence viewing $(y_1,\ldots,y_k)$ as a single-district plurality election, it follows from \Cref{lem:pluralitywithnoise} that sampling $\frac{3}{(\frac{3\eps}{32}-\gamma)^2}\log \frac{4}{\delta}$ districts from the distribution $U$ without replacement would predict the winner with probability at least $1-\frac{\delta}{2}$. Hence overall, the probability of correctly predicting the winner is at least $1-\delta$.

\begin{theorem}
\label{lem:pluralitywithnoisemultiple}
There exists an algorithm for $(\epsilon,\delta,\gamma)-${\sc{Winner-Prediction}} for the plurality rule with sample complexity $\OO\left(\frac{1}{(\eps - \gamma)^4}\log \frac{1}{\eps} \log \frac{1}{\delta}\right)$.
\end{theorem}
\section{Estimating Margin of Victory}
\label{sec:estimatingMOVwithinadditiveerror}
For this section, we use $\gamma N$ to denote the margin of victory of a district-based election; \eps will be used for denoting error bounds. We first present an algorithm for $(\eps,\delta)-${\sc{MOV-Additive}} for the district-level plurality election. This gives an estimate of the margin of victory within an additive $\eps N$ error. We then bootstrap our algorithm to get an estimate of the margin of victory within a multiplicative error of $1\pm \eps$. Our algorithm for $(\eps,\delta)-${\sc{MOV-Multiplicative}} in fact works for any voting rule for which there exists an algorithm for $(\eps,\delta)-${\sc{MOV-Additive}}.

\subsection{Estimating MOV within additive error bounds}
\label{sec:estimatemovadditive}
We consider the district-level plurality election with 2 candidates A and B. Wlog we assume that A is the true winner of the election. As before, $n=N/k$ denotes the average population of a district and we assume that the population of each district is bounded by $\kappa n$, for some constant $\kappa \ge 2$.

\begin{algorithm}[H]
 \caption{}
  \begin{algorithmic}[1]
    \State{Sample $l_1 = \frac{27\kappa^2}{\epsilon^4}\log \frac{16}{\delta}$ districts from $D$ uniformly at random with replacement.}
    \State{In each of the sampled districts, sample $l_2 = \frac{27\kappa^2}{\epsilon^2}\log \frac{8l_1}{\delta}$ votes uniformly at random with replacement and predict the number of votes received by A and B. }
    \Return the margin of victory of the sampled election.
    
    %Let $\CC\in \{A,B\}$ be the predicted winner in $(\frac{1}{2}+\mu)l_1$ districts (call this set $T^{\CC}$) for some $\mu \ge 0$.
    %\State{Let $|\text{MOV}'(E_d)|$ denote the absolute value of the predicted margin of victory in district $d$. If $\CC$ is the predicted winner in $d$, we let $\text{MOV}'(E_d) \ge 0$; else let $\text{MOV}'(E_d)<0$. Let $d_{j_1},\ldots,d_{j_{(\frac{1}{2}+\mu)l_1}}$ be the districts of $T^{\CC}$ arranged in non-decreasing order of $\text{MOV}'(E_d)$.}
    %\Return $\frac{p}{l_1}\cdot \sum\limits_{i=1}^{\mu l_1} \text{MOV}'(E_{d_{j_i}})$.
    %\Require{$(V, (X_v)_{v\in V}, \gamma, (m_v)_{v\in V}, k)$}
   \end{algorithmic}
\label{alg:alg30}
\end{algorithm}

It is easily seen that $l_1 = \OO\left(\frac{1}{\epsilon^4}\log \frac{1}{\delta}\right)$ and $l_2=\OO\left(\frac{1}{\epsilon^2}\log \frac{1}{\epsilon \delta}\right)$. Thus the sample complexity of the above algorithm is $\OO\left(\frac{1}{\eps^6}\log \frac{1}{\eps \delta}\log \frac{1}{\delta}\right)$.

\begin{lemma}
\label{lem:samplecomplexityformovestimation}
\Cref{alg:alg30} uses $\OO\left(\frac{1}{\eps^6}\log \frac{1}{\eps \delta}\log \frac{1}{\delta}\right)$ samples.
\end{lemma}

As in \Cref{sec:boundedpopulation} and \Cref{sec:winnerpredictwitharbitrarypop}, let $D^A$ (resp. $D^B$) be the set of districts where A (resp. B) wins and let $\left|D^A \right|=(\frac{1}{2}+\nu)k$ (assume $k$ is even so that $\nu k$ is an integer). For each $d\in D^A$, let $\text{MOV}(E_d)$ be the minimum number of votes to be changed in district $d$ in order to make B the winner of that district. Wlog let $d_{1},\ldots,d_{\left(\frac{1}{2}+\nu \right)p}$ be the districts of $D^A$ arranged in non-decreasing order of $\text{MOV}(E_d)$. Let $D^A_1 = \{d_1,\ldots,d_{\nu p+1} \}$ and $D^A_2 = D^A \setminus D^A_1$. Since $\text{MOV}(E)=\gamma N$, we have $\gamma N \le \sum\limits_{d\in D_1^A} \text{MOV}(E_{d}) \le \gamma N+\kappa n$.

As in the analyses of \Cref{alg:alg1} and \Cref{thm:epsfourlowerbound}, for estimating the probability that the returned estimate lies in the range $[(\gamma - \epsilon)N,(\gamma + \epsilon)N]$, \Cref{alg:alg20} may be viewed as first sampling $l_2$ votes in each district and predicting the number of votes received by A and B in each, and then sampling $l_1$ districts uniformly at random with replacement. 

Let $T_1,T_2,T_3$ respectively be the set of districts sampled from $D^A_1,D^A_2,D^B$. Let $\left |T_i \right|=k_i$ so that $k_1 + k_2 + k_3 =l_1$. Observe that $T_1$ (resp. $T_2,T_3$) can be thought of as a uniform random sample of $k_1$ (resp. $k_2,k_3$) districts from $D^A_1$ (resp. $D^A_2,D^B$).

We now define the following events.

$\MM_i$ : $\left |\frac{k_i}{l_1} - \frac{\left |D^A_i \right |}{k} \right | \le \frac{\epsilon^2}{3\kappa}$, $i \in [2]$.

$\MM_3$ : $\left |\frac{k_3}{l_1} - \frac{\left |D^B \right |}{k}\right | \le \frac{\epsilon^2}{3\kappa}$.

$\MM_4$ : In each district, the predicted fraction of votes received by A (and therefore B) lies within an additive error of less than $\epsilon/3\kappa$ from the true fraction of votes received by A (resp. B) in that district.

$\MM_5$ : $\left |\frac{k_1}{\left |D^A_1 \right |}-\frac{l_1}{k}\right |\le \frac{\epsilon l_1}{10k}$.

From \Cref{thm:pluralitythm}, it directly follows that $\text{Pr}\left(\bigcap\limits_{i=1}^{3} \MM_i \right)\ge 1-\frac{\delta}{8}$ and $\text{Pr}\left(\MM_4 \right) \ge 1-\frac{\delta}{4}$. Thus $\text{Pr}\left(\bigcap\limits_{i=1}^{4} \MM_i \right)\ge 1-\left(\frac{\delta}{8}-\frac{\delta}{4}\right)=1-\frac{3\delta}{8}$. We assume throughout that the event $\bigcap\limits_{i=1}^{4} \MM_i$ holds true.

Let $e$ denote the estimate returned by \Cref{alg:alg20}. Let $e_i = \frac{k}{l_1}\cdot \sum\limits_{d\in T_i} \text{MOV}'(E_d)$, $i\in [3]$. Thus $e=e_1 +e_2 +e_3$.

Let $\CC\in \{A,B\}$ be the winner of the sampled election. In the following four lemmas, we show that $e\in [(\gamma - \eps)N, (\gamma + \eps)N]$ with high probability.

\begin{lemma}
\label{lem:cequalsa}
If $\CC=A$ and $\epsilon$ is sufficiently small, then $e\le (\gamma + \epsilon)N$ with probability at least $1-\delta$.
\end{lemma}
\begin{proof}
Since $\MM_2$ holds and $\text{MOV}'(E_d)\le \kappa n$, $\forall d\in D^A_2$, we have $e_2 \le \frac{k}{l_1}\cdot \frac{\epsilon^2 l_1}{3\kappa}\cdot \kappa n = \frac{\epsilon^2 N}{3} \le \frac{\epsilon N}{18}$, for sufficiently small $\epsilon$. Similarly, since $\MM_4$ holds, if A is predicted as the winner in any district of $D^B$, the maximum possible Margin of Victory of A would be $\frac{\epsilon}{3\kappa}\cdot \kappa n=\frac{\epsilon n}{3}$. Thus $e_3 \le \frac{k}{l_1}\cdot l_1 \cdot \frac{\epsilon n}{3}=\frac{\epsilon N}{3}$.

We now upper bound $e_1$. If $\nu \le \epsilon/2\kappa$, we have $e_1 \le \frac{k}{l_1}\cdot \left(\frac{\epsilon}{2\kappa}+ \frac{1}{k}+\frac{\epsilon^2}{3\kappa}\right)l_1 \cdot \kappa n \le \frac{11\epsilon N}{18}$, for sufficiently small $\epsilon$. Hence the estimate $e = e_1 + e_2 + e_3\le \frac{11\epsilon N}{18} + \frac{\epsilon N}{18} +\frac{\epsilon N}{3}< (\gamma + \epsilon)N$.

Now let $\nu > \epsilon/2\kappa$. Clearly $\mathbb{E}[k_1]=\frac{l_1}{k}\cdot \left |D^A_1 \right|>\nu l_1$. Applying the multiplicative form of Chernoff bound (\Cref{thm:chernoff}) with $\theta = \epsilon/10$, we get $\text{Pr}(\overline{\MM}_5) \le 2e^{-\frac{\epsilon^2}{300} \cdot \nu l_1}\le 2e^{-\frac{\epsilon^3}{600\kappa} l_1}\le \delta/4$ (since $l_1 = \frac{27\kappa^2}{\epsilon^4}\log \frac{16}{\delta} \ge \frac{600\kappa}{\epsilon^3}\log \frac{8}{\delta}$, for sufficiently small $\epsilon$). Notice that for $e_1$ to be maximum, A must be declared the winner in each district of $T_1$. Let $U_j$ be the predicted Margin of Victory of A in the $j^{\text{th}}$ district of $T_1$, $j\in [k_1]$. Let $\MM_6$ denote the event that $\left |\frac{\sum\limits_{j=1}^{k_1} U_j}{k_1} - \frac{\sum\limits_{d\in D^A_1} \text{MOV}'(E_d)}{\nu p+1}\right| \le \frac{\epsilon n}{8\nu}$. Note that each $U_j \in [0,\kappa n]$. Also, conditioned on $\MM_1$, $k_1 \ge \left(\nu - \frac{\epsilon^2}{3\kappa}\right)l_1$. Since $\nu > \epsilon/2\kappa$, for sufficiently small $\epsilon$, we have $k_1 \ge \frac{\epsilon}{4\kappa} l_1$. Applying Hoeffding's inequality (\Cref{thm:hoeffding}) with $a=0$, $b=\kappa n/2$ and $\theta = \epsilon n/8\nu$, we have $\text{Pr}(\overline{\MM}_6|\MM_1)\le 2e^{-\frac{\epsilon^2}{32\kappa^2 \nu^2}\cdot k_1}\le 2e^{-\frac{\epsilon^3}{32\kappa^3}\cdot l_1}\le \delta/4$, since $l_1 = \frac{27\kappa^2}{\epsilon^4}\log \frac{16}{\delta} \ge \frac{32\kappa^3}{\epsilon^3}\log \frac{8}{\delta}$. Thus $\text{Pr}(\MM_6)\ge \text{Pr}(\MM_6|\MM_1)\text{Pr}(\MM_1)\ge \left(1-\frac{\delta}{4}\right) \left(1-\frac{\delta}{8}\right)\ge 1-\frac{3\delta}{8}$.

Finally assume that $\MM_5 \cap \MM_6$ holds true. This happens with probability at least $1-\left(\frac{\delta}{4}+\frac{3\delta}{8}\right)=1-\frac{5\delta}{8}$. Now $\MM_5$ implies that $\frac{k}{l_1}\le \left(1+\frac{\epsilon}{10}\right)\frac{\nu k}{k_1}$. Also $\MM_4$ implies that $\sum\limits_{d\in D^A_1} \text{MOV}'(E_d) \le \sum\limits_{d\in D^A_1} \text{MOV}(E_d) + \frac{\epsilon}{3\kappa}\cdot \kappa n \cdot (\nu k+1) \le \gamma N + \kappa n +\frac{\epsilon N}{3} + \frac{\eps n}{3}$. Hence $e_1 = \frac{k}{l_1}\cdot \sum\limits_{j=1}^{k_1} U_j \le \left(1+\frac{\epsilon}{10}\right)\frac{\nu k}{k_1}\cdot \sum\limits_{j=1}^{k_1} U_j \le \left(1+\frac{\epsilon}{10}\right)\left(\sum\limits_{d\in D^A_1} \text{MOV}'(E_d) + \frac{\epsilon N}{8} + \frac{\eps n}{8\nu}\right)\le \left(1+\frac{\epsilon}{10}\right)\left(\gamma N + \frac{\epsilon N}{2}\right)\le \left(\gamma + \frac{11\epsilon}{18}\right)N$ (as $\gamma \le 1/2$ and $\epsilon$ is sufficiently small). Thus with probability at least $1-\left(\frac{3\delta}{8}+\frac{5\delta}{8}\right)=1-\delta$, the estimate returned $e = e_1 + e_2 + e_3 \le (\gamma + \frac{11\epsilon}{18})N + \frac{\epsilon N}{40} +  \frac{\epsilon N}{3}=(\gamma + \epsilon)N$.
\end{proof}

\begin{lemma}
\label{lem:cequalsalower}
If $\CC=A$ and $\epsilon$ is sufficiently small, then $e\ge (\gamma - \epsilon)N$ with probability at least $1-\delta$.
\end{lemma}
\begin{proof}
If $\gamma < \epsilon$, we are done since $e\ge 0 \ge (\gamma - \epsilon)N$. Hence assume that $\gamma \ge \epsilon$. Since $\text{MOV}(E_d)\le \kappa n$, $\forall d\in D^A$ and $\sum\limits_{d\in D^A_1} \text{MOV}(E_d) \ge \gamma N\ge \epsilon N$, we have $\left |D^A_1 \right |\ge \epsilon k/\kappa$. As in the proof of Lemma \ref{lem:cequalsa}, Chernoff bound would give $\text{Pr}(\MM_5)\ge 1-\frac{\delta}{4}$. Observe that $e$ will be minimised when in each district of $D^A$, the true fraction of votes received by A exceeds the predicted fraction by $\epsilon/3\kappa$. Let $\overline{T}_1\subset T_1$ be the set of districts of $T_1$ where B is the predicted winner. Let $D'$ be the set of sampled districts that are considered in the computation of the margin of victory in step 3 of \Cref{alg:alg20} and let $T'_1 = \overline{T}_1 \cup D'$ and let $\left |T'_1 \right |=k'_1$. Let $\UU_j$ be the predicted Margin of Victory of A in the $j^{\text{th}}$ district of $T'_1$, where $j\in [k'_1]$ (note that $\UU_j$ is negative in the districts of $\overline{T}_1$). Let $\MM_7$ denote the event that $\left |\frac{\sum\limits_{j=1}^{k'_1} \UU_j}{k'_1} - \frac{\sum\limits_{d\in D^A_1} \text{MOV}'(E_d)}{\nu k+1}\right | \le \frac{\epsilon n}{8\nu}$ (note that here $\text{MOV}'(E_d)$ could be negative). Again similar to the proof of Lemma \ref{lem:cequalsa}, Hoeffding's inequality gives $\text{Pr}(\MM_7|\MM_1)\ge 1- \frac{\delta}{4}$ and therefore $\text{Pr}(\MM_7) \ge 1-\frac{3\delta}{8}$. Therefore $\text{Pr}(\MM_5 \cap \MM_7) \ge 1-\frac{5\delta}{8}$.

Now $\MM_1$ implies that $k_1 \ge \left(\nu - \frac{\epsilon^2}{2\kappa}\right)l_1 \ge \frac{15\epsilon}{8\kappa}l_1$. Since $\MM_2$ implies that $k_2 \ge \left(\frac{1}{2}- \frac{\epsilon^2}{2\kappa}\right)l_1$, it follows that $k'_1 \ge k_1 - \frac{\epsilon^2}{2\kappa}l_1\ge k_1 - \frac{4\epsilon}{15}k_1 = \left(1-\frac{4\epsilon}{15}\right)k_1$. Thus assuming $\MM_5 \cap \MM_7$ holds, $e_1 \ge \frac{k}{l_1}\cdot \sum\limits_{j=1}^{k'_1} \UU_j \ge \left(1-\frac{\epsilon}{10}\right)\left(1-\frac{4\epsilon}{15}\right)\left(\sum\limits_{d\in D^A_1} \text{MOV}'(E_d) - \frac{\epsilon N}{8}-\frac{\eps n}{8\nu}\right)\ge \left(1-\frac{11\epsilon}{30}\right)\left(\gamma N - \frac{\epsilon N}{3} - \frac{\eps n}{3}-\frac{\epsilon N}{8}-\frac{\eps n}{8\nu}\right)> (\gamma - \epsilon)N$. Hence with probability at least $1-\delta$, the estimate $e\ge e_1 \ge (\gamma - \epsilon)N$.
\end{proof}

\begin{lemma}
\label{lem:cequalsbupper}
If $\CC=B$ and $\epsilon$ is sufficiently small, then $e\le \epsilon N$ with probability at least $1-\frac{3\delta}{8}$.
\end{lemma}
\begin{proof}
Since $\MM_4$ holds, if B is predicted as the winner in any district of $D^A$, then the maximum possible Margin of Victory of B would be $\frac{\epsilon}{3\kappa}\cdot \kappa n=\frac{\epsilon n}{3}$. Thus $e_1 + e_2 \le \frac{k}{l_1}\cdot l_1 \cdot \frac{\epsilon n}{3} = \frac{\epsilon N}{3}$. Again since $\MM_3$ holds, $k_3 \le \left(\frac{1}{2}+\frac{\epsilon^2}{3\kappa}\right)l_1$. Thus $e_3\le \frac{k}{l_1}\cdot \frac{\epsilon^2}{3\kappa}l_1 \cdot \kappa n = \frac{\kappa \epsilon^2 N}{3} \le \frac{2\epsilon N}{3}$, for sufficiently small $\epsilon$. Hence with probability at least $1-\frac{3\delta}{8}$, $e=e_1 + e_2 + e_3 \le \frac{\epsilon N}{3} + \frac{2\epsilon N}{3}=\epsilon N$.
\end{proof}

\begin{lemma}
\label{lem:cequalsblower}
If $\CC=B$, then $\gamma < \epsilon$ with probability at least $1-\frac{3\delta}{8}$ and thus $e\ge 0 > (\gamma - \epsilon)N$.
\end{lemma}
\begin{proof}
Suppose $\gamma \ge \epsilon$. Then from Lemma \ref{lem:movlemma}, there exist at least $\left(\frac{1}{2}+\frac{\epsilon}{\kappa}\right)k$ districts where A receives at least $\frac{1}{2}+\frac{\epsilon}{3\kappa}$ fraction of votes. Thus $\nu \ge \epsilon/3\kappa$. Since $\MM_4$ holds, the predicted fraction of votes received by A in each of these districts is more than $\frac{1}{2}+\frac{\epsilon}{3\kappa} - \frac{\epsilon}{3\kappa}=\frac{1}{2}$. Also since $\MM_1$ and $\MM_2$ hold, $k_1 \ge \left(\nu - \frac{\epsilon^2}{2\kappa}\right)l_1 \ge \left(\frac{\epsilon}{\kappa} - \frac{\epsilon^2}{2\kappa}\right)l_1$ and $k_2 \ge \left(\frac{1}{2}-\frac{\epsilon^2}{2\kappa}\right)l_1$ and  and therefore $k_1+k_2 > \frac{1}{2}l_1$. This contradicts the fact that B is the predicted winner in more than half of the sampled districts. Thus with probability at least $1-\frac{3\delta}{8}$, $\gamma < \epsilon$. 
\end{proof}

Combining Lemma \ref{lem:cequalsa}, Lemma \ref{lem:cequalsalower}, Lemma \ref{lem:cequalsbupper} and Lemma \ref{lem:cequalsblower}, it follows that the estimate returned by \Cref{alg:alg20} lies in the range $[(\gamma - \epsilon)N,(\gamma + \epsilon)N]$ with probability at least $1-\delta$.

\begin{theorem}
\label{thm:movestimatefordistrictplurality}
There exists an algorithm for $(\eps,\delta)-${\sc{MOV-Additive}} with sample complexity $\OO\left(\frac{1}{\eps^6}\log \frac{1}{\eps \delta}\log \frac{1}{\delta}\right)$ for the district-level plurality election.
\end{theorem}

\subsection{Estimating MOV within multiplicative error bounds}
%\section{Almost perfect estimation of MOV}
%\label{sec:perfectmov}
We now present a black-box algorithm that, given any algorithm for the $(\eps,\delta)-${\sc{MOV-Additive}} problem, returns an estimate of the margin of victory within a multiplicative error of $1 \pm \epsilon$ with high probability.

Let $E$ be any arbitrary election with $m$ candidates and $N$ voters. Suppose there exists an algorithm $\AA$ for the $(\eps,\delta)-${\sc{MOV-Additive}} problem; thus for any $\epsilon,\delta >0$, $\text{Pr}(|\AA(\epsilon,\delta)-\gamma N| \le \epsilon N) \ge 1-\delta$. Using $\AA$ as a black box, we design an algorithm for the $(\eps,\delta)-${\sc{MOV-Multiplicative}} problem. Our algorithm is as follows.

%Suppose there exists an algorithm $\AA$ that on inputs $\gamma$ and $\delta$ returns an estimate $\AA(\gamma,\delta)$ such that $\text{Pr}(|\AA(\gamma,\delta)-\epsilon N| \le \gamma N) \ge 1-\delta$. Using $\AA$ as a black box, we design an algorithm to estimate $\text{MOV}(E)$ within a multiplicative error of $1+\gamma$, for any $\gamma > 0$, with high probability. Our algorithm is as follows:

\begin{algorithm}[H]
 \caption{}
  \begin{algorithmic}[1]
    %\Require{$(V, (X_v)_{v\in V}, \gamma, (m_v)_{v\in V}, k)$}
    \For{$i=1,2,\ldots, \log_{1+\epsilon}N$}
        \State{$e_i \leftarrow \AA\left(\frac{1}{(1+\epsilon)^i},\frac{\delta}{2^i}\right)$, $\lambda_i \leftarrow \frac{(1+\epsilon)^{\log \frac{1}{\epsilon}/\log (1+\epsilon)}+1}{(1+\epsilon)^i}$}
        \State{If $e_i \ge \lambda_i N$, \textbf{return} $e_i$.}
    \EndFor    
    \Return $1$.
  \end{algorithmic}
\label{alg:alg5}
\end{algorithm}

We show that with high probability the estimate $e_i$ would be less than $\lambda_i N$, when $i< \OO\left(\frac{1}{\epsilon}\log \frac{1}{\epsilon \gamma}\right)$ and would exceed $\lambda_i N$ when $i\ge \Omega\left(\frac{1}{\epsilon}\log \frac{1}{\epsilon \gamma}\right)$. Thus \Cref{alg:alg5} would return for some $i=\Theta\left(\frac{1}{\epsilon}\log \frac{1}{\epsilon \gamma}\right)$. Using this, it can be shown that the estimate returned must lie in the range $[(1-\epsilon)\gamma N,(1+\epsilon)\gamma N]$ with high probability.

We now formalise the above notion. Let $\psi_1=\frac{\log \frac{1}{\epsilon}}{\log (1+\epsilon)}$ and $\psi_2=\frac{\log \frac{2}{\epsilon}}{\log (1+\epsilon)}$. Let $p$ be the unique positive integer such that $\frac{1}{(1+\epsilon)^{p+1}}\le \gamma < \frac{1}{(1+\epsilon)^{p}}$. Let $\KK_i$ denote the event that $\left|\AA\left(\frac{1}{(1+\epsilon)^i},\frac{\delta}{2^i}\right) - \gamma N \right| \le \frac{N}{(1+\epsilon)^i}$. Thus $\text{Pr}(\KK_i) \ge 1- \frac{\delta}{2^i}$. Now we show the following two results.

\begin{lemma}
\label{lem:kilowerbound}
If $\KK_i$ holds, $e_i < \lambda_i N$, for $i\le p+\psi_1$.
\end{lemma}
\begin{proof}
For $i\le p+\psi_1$, we have $e_i \le \gamma N + \frac{N}{(1+\epsilon)^i}\le \frac{N}{(1+\epsilon)^p} + \frac{N}{(1+\epsilon)^i} \le \frac{N}{(1+\epsilon)^{i-\psi_1}} + \frac{N}{(1+\epsilon)^i} = \frac{(1+\epsilon)^{\psi_1}+1}{(1+\epsilon)^i}N = \lambda_i N$.
\end{proof}

\begin{lemma}
\label{lem:kiupperbound}
If $\KK_i$ holds, $e_i \ge \lambda_i N$, for $i\ge p+\psi_2$.
\end{lemma}
\begin{proof}
For $i\ge p+\psi_2$, $e_i \ge \gamma N - \frac{N}{(1+\epsilon)^i}\ge \frac{N}{(1+\epsilon)^{p+1}} - \frac{N}{(1+\epsilon)^i} \ge \frac{N}{(1+\epsilon)^{i-\psi_2+1}} - \frac{N}{(1+\epsilon)^i} = \frac{(1+\epsilon)^{\psi_2-1}-1}{(1+\epsilon)^i}N \ge \lambda_i N$, for sufficiently small $\epsilon$.
\end{proof}

Now let $\KK = \bigcap\limits_{i=1}^{p+\psi_2} \KK_i$. Since $\text{Pr}(\overline{\KK}_i)\le \frac{\delta}{2^i}$, by union bound, the probability that at least one of the events $\KK_1,\ldots,\KK_{p+\psi_2}$ does not hold is at most $\sum\limits_{i=1}^{p+\psi_2} \frac{\delta}{2^i}\le \sum\limits_{i=1}^{\infty} \frac{\delta}{2^i} = \delta$. Thus $\text{Pr}(\KK)\ge 1-\delta$. Now if $\KK$ holds, \Cref{alg:alg5} returns an $e_i$, for $i$ lying in the range $\{p+\psi_1+1,\ldots,p+\psi_2\}$. 

Let $\sigma$ denote the estimate returned by the algorithm. Then $\sigma \le \gamma N + \frac{N}{(1+\epsilon)^{p+\psi_1+1}}\le \left(1+\frac{1}{(1+\epsilon)^{\psi_1}}\right)\gamma N = (1+\epsilon)\gamma N$. Again $\sigma \ge \gamma N - \frac{1}{(1+\epsilon)^{p+\psi_2}}N\ge \left(1-\frac{1}{(1+\epsilon)^{\psi_2}}\right)\gamma N = \left(1-\frac{\epsilon}{2}\right)\gamma N \ge (1-\epsilon)\gamma N$. We thus have the following result.

\begin{lemma}
\label{lem:pklusmu}
\Cref{alg:alg5} returns an estimate in the range $[(1-\eps)\gamma N, (1+\eps)\gamma N]$ with probability at least $1-\delta$.
\end{lemma}

%\iffalse
%For many voting rules, the Margin of Victory can be estimated within an additive error of $\epsilon N$ with probability at least $1-\delta$ using small number of samples. The two voting rules considered in this paper, viz. the plurality rule (single-district case) and the median rule are two such examples (\Cref{thm:estimatemov} and \Cref{thm:estimatemovmedian} respectively). The following result can be shown for any such voting rule (see \Cref{sec:omittedproofsofmedian}).
%
%\begin{theorem}
%\label{thm:samplecomplexityepsiloncube}
%If $\AA$ uses at most $O(\frac{1}{\gamma^2}\log \frac{1}{\delta})$ samples on inputs $\gamma$ and $\delta$, \Cref{alg:alg5} uses at most $O(\frac{1}{\epsilon^2}\frac{1}{\gamma^3}(\frac{1}{\gamma}\log \frac{1}{\epsilon \gamma}+\log \frac{1}{\delta}))$ samples in expectation.
%\end{theorem}
%\fi

Combining \Cref{thm:movestimatefordistrictplurality} and \Cref{lem:pklusmu}, we have the following result.
%\Cref{alg:alg20} estimates the Margin of Victory of the district-level plurality election with 2 candidates, within an additive $\epsilon N$ error with probability at least $1-\delta$ using $O(\frac{1}{\epsilon^6}\log \frac{1}{\epsilon \delta}\log \frac{1}{\delta})$ samples (\Cref{thm:movestimatefordistrictplurality}). We thus have the following result.

\begin{theorem}
\label{thm:lastthmijcai}
There exists an algorithm for $(\eps,\delta)-${\sc{MOV-Multiplicative}} with expected sample complexity $\OO\left(\frac{1}{\epsilon^7}\frac{1}{\gamma^6}\left(\frac{1}{\epsilon}\log \frac{1}{\epsilon \gamma}+ \log \frac{1}{\delta}\right)^2\right)$ for the district-level plurality election with 2 candidates when the population of each district is bounded by a constant times the average population of a district, where $\gamma N$ is the (unknown) margin of victory of the election.
\end{theorem}

%We first show that if \Cref{alg:alg5} halts when $i=O(\frac{1}{\epsilon}\log \frac{1}{\epsilon \gamma})+j$, then the number of samples used is at most $O(\frac{1}{\epsilon^7}\frac{1}{\gamma^6}(\frac{1}{\epsilon}\log \frac{1}{\epsilon \gamma}+ \log \frac{1}{\delta}+j)^2)$ (denote this by $\varphi(j)$). Next since the probability that \Cref{alg:alg5} does not halt when $i=O(\frac{1}{\epsilon}\log \frac{1}{\gamma \epsilon})+j$ is at most $\frac{\delta}{2^{O(\frac{1}{\epsilon}\log \frac{1}{\gamma \epsilon})+j}}$ (which is exponentially small in $j$), we show that the expected sample complexity is dominated by $\varphi(0)=O(\frac{1}{\epsilon^7}\frac{1}{\gamma^6}(\frac{1}{\epsilon}\log \frac{1}{\epsilon \gamma}+ \log \frac{1}{\delta})^2)$. Hence the result follows.
\begin{proof}
From \Cref{lem:pklusmu}, the estimate returned by \Cref{alg:alg5} lies in the range $[(1-\eps)\gamma N, (1+\eps)\gamma N]$ with probability at least $1-\delta$. We therefore need to only bound the sample complexity.

For a particular value of $i$, it is clear that \Cref{alg:alg5} uses at most $\OO\left((1+\epsilon)^{6i}\left(i+\log \frac{1}{\delta}\right)^2\right)$ samples. Let $p$ be as defined as before. Let $\varphi(j)$ denote the number of samples collected if \Cref{alg:alg5} halts when $i=p+\psi_2+j$. Then $\varphi(j) = \sum\limits_{i=1}^{p+\psi_2+j} \OO\left((1+\epsilon)^{6i}\left(i+\log \frac{1}{\delta}\right)^2\right) \le \OO\left(\frac{1}{\epsilon}(1+\epsilon)^{6(p+\psi_2+j)}\left(p+\psi_2+j+\log \frac{1}{\delta}\right)^2\right)$. For $j\ge 1$, let $M_j$ denote the event that \Cref{alg:alg5} halts when $i=p+\psi_2+j$ and let $M_0$ denote the event that \Cref{alg:alg5} halts for some $i\le p+\psi_2$. Then for $j\ge 1$, $\text{Pr}(M_j) \le \text{Pr(\Cref{alg:alg5} does not halt when } i = p+\psi_2+j-1) \le \frac{\delta}{2^{p+\psi_2+j-1}}$. Also trivially $\text{Pr}(M_0) \le 1$. Hence the expected sample complexity is at most $\sum\limits_{j=0}^{\infty} \varphi(j)\text{Pr}(M_j) \le \varphi(0) + \sum\limits_{j=1}^{\infty} \OO\left(\frac{1}{\epsilon}(1+\epsilon)^{6(p+\psi_2+j)}\left(p+\psi_2+j+\log \frac{1}{\delta}\right)^2\cdot \frac{\delta}{2^{p+\psi_2+j-1}}\right)$. Now $\varphi(0)=\OO\left(\frac{1}{\epsilon}(1+\epsilon)^{6(p+\psi_2)}\left(p+\psi_2+\log \frac{1}{\delta}\right)^2\right)$, while the second term in the sum is at most $\OO\left(\frac{1}{\epsilon}(1+\epsilon)^{6(p+\psi_2)}\left(p+\psi_2+\log \frac{1}{\delta}\right)^2 \cdot \frac{\delta}{2^{p+\psi_2}}\right)$. Hence the overall expected sample complexity is bounded by $\OO\left(\frac{1}{\epsilon}(1+\epsilon)^{6(p+\psi_2)}\left(p+\psi_2+\log \frac{1}{\delta}\right)^2\right) = \OO\left(\frac{1}{\epsilon^7}\frac{1}{\gamma^6}\left(\frac{1}{\epsilon}\log \frac{1}{\epsilon \gamma} + \log \frac{1}{\delta}\right)^2\right)$.
\end{proof}

\section{Conclusion and Future Work}
We have initiated the study of the sample complexity for predicting the winner in a district-based election. We have shown some preliminary results for the problem for some voting rules. We believe that the problem and our results are both practically and theoretically interesting. An important future direction of research is to find the sample-complexity lower bounds for various voting rules. Some of our lower bounds work only for some class of algorithms. Also, extending our algorithm for winner prediction when the margin of victory is not known, to arbitrary number of candidates is an important future direction of research; our algorithm works only for two candidates.

\subsection*{Acknowledgement}

Palash Dey is partially supported by DST INSPIRE grant DST/INSPIRE/04/2016/001479 and ISIRD grant of IIT Kharagpur. Swagato Sanyal is supported by an ISIRD grant by Sponsored Research and Industrial Consultancy, IIT Kharagpur.

\bibliographystyle{plain} % Use the "custom" BibTeX style for formatting the Bibliography
\bibliography{ref}
%\input{problems}
%\input{results_cleaned_up}

%\input{abstract}
%\input{introduction}
%\input{contributions}
%\input{prelim}
%\input{results-new}

%\input{results-old}

%\bibliographystyle{unsrt}
%\bibliography{ref}

\end{document}